\documentclass{article} %
\usepackage{iclr2023_conference,times}

\usepackage{amsmath,amsfonts,bm}

\def\eqref#1{equation~\ref{#1}}
\def\Eqref#1{Equation~\ref{#1}}

\def\1{\bm{1}}

\newcommand{\planardyadic}{\mathbf{P}}
\newcommand{\inertiadyadic}{\mathbf{I}}
\newcommand{\code}[1]{{\small \texttt{#1}}}

\def\rvp{{\mathbf{p}}}

\def\rvx{{\mathbf{x}}}

\def\vtheta{{\bm{\theta}}}
\def\va{{\bm{a}}}
\def\vb{{\bm{b}}}
\def\vc{{\bm{c}}}

\def\vh{{\bm{h}}}

\def\vm{{\bm{m}}}

\def\vo{{\bm{o}}}
\def\vp{{\bm{p}}}

\def\vt{{\bm{t}}}
\def\vu{{\bm{u}}}
\def\vv{{\bm{v}}}

\def\vx{{\bm{x}}}

\def\vz{{\bm{z}}}

\def\mA{{\bm{A}}}

\def\mC{{\bm{C}}}

\def\mH{{\bm{H}}}
\def\mI{{\bm{I}}}

\def\mR{{\bm{R}}}
\def\mS{{\bm{S}}}
\def\mT{{\bm{T}}}

\def\mV{{\bm{V}}}

\def\mX{{\bm{X}}}

\def\mZ{{\bm{Z}}}

\def\mPhi{{\bm{\Phi}}}

\DeclareMathAlphabet{\mathsfit}{\encodingdefault}{\sfdefault}{m}{sl}
\SetMathAlphabet{\mathsfit}{bold}{\encodingdefault}{\sfdefault}{bx}{n}

\def\sD{{\mathbb{D}}}

\def\sP{{\mathbb{P}}}

\def\sT{{\mathbb{T}}}
\def\sU{{\mathbb{U}}}

\newcommand{\E}{\mathbb{E}}

\newcommand{\R}{\mathbb{R}}

\usepackage{amsthm}
\newtheorem{proposition}{Proposition}

\newcommand{\vepsilon}{\bm{\varepsilon}}

\newcommand{\vlambda}{\bm{\lambda}}
\newcommand{\mPi}{\bm{\Pi}}

\DeclareMathOperator{\diag}{diag}
\DeclareMathOperator{\vectorize}{vec}

\definecolor{mydarkblue}{rgb}{0,0.08,0.45}
\usepackage[colorlinks=true,linkcolor=mydarkblue,urlcolor=mydarkblue,citecolor=mydarkblue,filecolor=mydarkblue]{hyperref}
\usepackage{xurl}
\usepackage{cancel}
\usepackage{graphicx}
\usepackage{subcaption}

\usepackage{microtype}
\usepackage[english]{babel}
\usepackage{booktabs}
\usepackage{bbm}
\usepackage{amssymb}
\usepackage{xcolor}
\usepackage{soul}

\usepackage{algorithm}
\usepackage{algpseudocode}
\usepackage{algorithmicx}

\usepackage{multicol}
\usepackage{multirow}

\title{Reflection-Equivariant Diffusion for 3D \\ Structure Determination from Isotopologue \\ Rotational Spectra in Natural Abundance}

\author{Austin H. Cheng$^{1,2,3}$\thanks{Equal contribution. \texttt{<austin@cs.toronto.edu, alston.lo@mail.utoronto.ca>}} \ \ Alston Lo$^{2,3*}$ \ \ Santiago Miret$^{4}$ \ \ Brooks H. Pate$^{5}$ \ \ Alán Aspuru-Guzik$^{1,2,3}$\thanks{Acceleration Consortium. Department of Chemical Engineering and Applied Chemistry, University of Toronto. Department of Materials Science and Engineering, University of Toronto. Lebovic Fellow, Canadian Institute for Advanced Research (CIFAR).} \\
    $^1$Department of Chemistry, University of Toronto  \\
    $^2$Department of Computer Science, University of Toronto  \\
    $^3$Vector Institute  \ \
    $^4$Intel Labs \ \
    $^5$Department of Chemistry, University of Virginia \\
    \url{https://github.com/aspuru-guzik-group/kreed}
}

\graphicspath{ {./figures/} }

\makeatletter
\newcommand{\linethrough}{\mathpalette\@thickbar}
\newcommand{\@thickbar}[2]{{#1\mkern0mu\vbox{
    \sbox\z@{$#1#2\mkern-1.5mu$}%
    \dimen@=\dimexpr\ht\tw@-\ht\z@+2\p@\relax %
    \hrule\@height0.5\p@ %
    \vskip\dimen@
    \box\z@}}
}
\makeatother

\newcommand{\name}{\textsc{Kreed}}
\newcommand{\namex}{(\underline{K}raitchman \underline{RE}flection-\underline{E}quivariant \underline{D}iffusion)}

\iclrfinalcopy %
\begin{document}

\maketitle
\begin{abstract}

Structure determination is necessary to identify unknown organic molecules, such as those in natural products, forensic samples, the interstellar medium, and laboratory syntheses.
Rotational spectroscopy enables structure determination by providing accurate 3D information about small organic molecules via their moments of inertia.
Using these moments, Kraitchman analysis determines isotopic substitution coordinates, which are the unsigned $|x|,|y|,|z|$ coordinates of all atoms with natural isotopic abundance, including carbon, nitrogen, and oxygen.
While unsigned substitution coordinates can verify guesses of structures, the missing $+/-$ signs make it challenging to determine the actual structure from the substitution coordinates alone.
To tackle this inverse problem, we develop \name{} \namex{}, a generative diffusion model that infers a molecule's complete 3D structure from its molecular formula, moments of inertia, and unsigned substitution coordinates of heavy atoms.
\name{}'s top-1 predictions identify the correct 3D structure with $>$98\% accuracy on the QM9 and GEOM datasets when provided with substitution coordinates of all heavy atoms with natural isotopic abundance.
When substitution coordinates are restricted to only a subset of carbons, accuracy is retained at 91\% on QM9 and 32\% on GEOM.
On a test set of experimentally measured substitution coordinates gathered from the literature, \name{} predicts the correct all-atom 3D structure in 25 of 33 cases, demonstrating experimental applicability for context-free 3D structure determination with rotational spectroscopy.

\end{abstract}

\section{Introduction}

Rotational spectroscopy provides rich 3D structural information about molecules via their principal moments of inertia $I_X, I_Y, I_Z$ \citep{gordy1984microwave, townes2013microwave}.
Since a small percentage of atoms exist as isotopes in natural abundance (e.g., about 1\% of carbon atoms are carbon-13), a chemical sample contains isotopically substituted versions of the molecule of interest. 
These isotopologues are chemically identical to the original parent molecule, but have an atom substituted with another isotope, so they have perturbed moments of inertia $I_X^*, I_Y^*, I_Z^*$.
The change in moments caused by the isotopic substitution reveals information about the position of the substituted atom.
Indeed, \cite{kraitchman1953determination} provides an expression which takes in parent and isotopologue moments and outputs the unsigned $|x|,|y|,|z|$ positions of the substituted atom in the isotopologue.
The coordinates are defined in the molecule's principal axis system, where the origin is the weighted center-of-mass and the axes are the principal axes of rotation.
This analysis can be repeated for all possible isotopologues to obtain the unsigned substitution coordinates for all atoms which have natural isotopic abundance, including carbon, nitrogen, and oxygen.

\begin{figure}[t]
\begin{center}
\includegraphics[width=\textwidth]{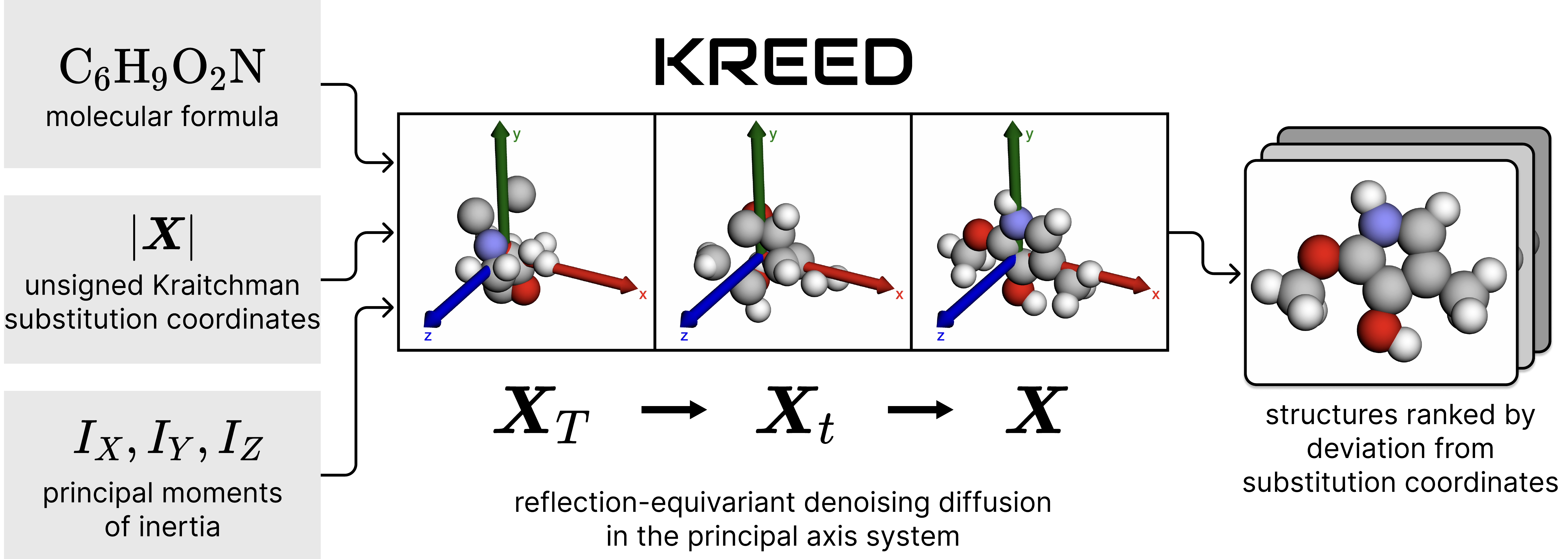}
\caption{\name{} takes as input molecular formula, Kraitchman's substitution coordinates, and principal moments of inertia \emph{(left)} and runs a learned reverse diffusion process of Euclidean steps in the molecule's principal axis system \emph{(center)} to obtain a ranked list of structures \emph{(right)}.}
\label{fig:overview}
\end{center}
\end{figure}

While unsigned substitution coordinates can easily verify and refine predicted 3D structures calculated from quantum chemistry \citep{brown2006rotational, shipman2011structure, seifert2015autofit}, it is a difficult problem to infer the full 3D structure from just the substitution coordinates, even if the molecular formula is given.
Since each atom lies in one of 8 octants, the measured substitution coordinates of a set of $m$ atoms are consistent with $8^{m -1}$ possible arrangements (fix the first atom), but only one of these makes up the true 3D structure.
Furthermore, substitution coordinates are not available for atoms without natural isotopic abundance (e.g., hydrogen, fluorine, and phosphorus).
On top of that, zero-point vibrational effects prevent accurate measurement of substitution coordinates lying close to a principal axis.

To solve this inverse problem, we develop \name{} \namex{}, an equivariant diffusion model that infers the equilibrium 3D structure of organic molecules given only the molecular formula, unsigned substitution coordinates, and principal moments of inertia (Figure \ref{fig:overview}).
We define a diffusion process where all atoms take Euclidean steps in the principal axis system of the molecule and train a denoiser network to reverse this diffusion conditioned on atom types, substitution coordinates, and moments of inertia. Since the principal axis system is defined up to sign-flips, we relax the equivariance of the denoiser from E(3)- to reflection-equivariance for axially-aligned reflections across the $xy,yz,xz$ planes.
We find it is not necessary to enforce strict adherence to the substitution coordinates during the diffusion process, so generated structures are instead ranked by deviation from substitution coordinates.
Because not all substitution coordinates are measurable, we perform data augmentation whereby a percentage of the substitution coordinates are randomly dropped during training.

Our approach enables context-free 3D structure determination, meaning that structure can be determined without any information of atom connectivity, SMILES string, or initial geometry.
Aside from the substitution coordinates determined from rotational spectroscopy, our method requires only the molecular formula, which can be determined with high-resolution mass spectrometry \citep{marshall2008high}.
Given this information, \name{} generates a ranked list of candidate all-atom 3D structures. Top-1 predictions identify the correct 3D structure with $>$98\% accuracy on the QM9 \citep{ramakrishnan2014quantum} and GEOM \citep{axelrod2022geom} datasets when given the substitution coordinates of all heavy atoms with natural isotopic abundance.
When provided with the substitution coordinates for only carbon, and only 90\% of them, top-1 predictions identify the correct 3D structure 91\% of the time for QM9 and 32\% of the time for GEOM.
We validate \name{} on a test set of experimentally measured substitution coordinates we gathered from the literature and find the model predicts the correct all-atom 3D structure in 25 of 33 cases, demonstrating potential for context-free 3D structure determination with rotational spectroscopy.

\clearpage
\section{Background and Problem Setup}

For a molecule whose 2D and 3D structure are both unknown, we wish to determine the all-atom 3D structure of its lowest energy conformer.
We begin by making five assumptions: 
\textbf{(1)} The molecule's formula can be determined from high resolution mass spectrometry \citep{marshall2008high}.
\textbf{(2)} The molecule's structure is well-approximated as a rigid rotor, so that its rotational spectrum is entirely described by its rotational constants $A, B, C$.
\textbf{(3)} The molecule is an asymmetric top (i.e.,  $A > B > C$) so that there are no molecular symmetries to exploit, which holds for the vast majority of molecules. 
\textbf{(4)} The molecule has a permanent dipole moment so that its rotational spectrum can be measured.
Finally, \textbf{(5)} the rotational constants of the molecule and many of its naturally-abundant isotopologues can be assigned in a context-free manner without any prior guess of 3D geometry.

While quickly and reliably assigning rotational constants from a given rotational spectrum is not a solved problem, significant progress has been made. Automated fitting tools such as AUTOFIT \citep{seifert2015autofit} and PGOPHER \citep{western2017pgopher, western2019automatic} have increased the speed of assigning spectra, while genetic algorithms \citep{leo2006application} and neural networks \citep{zaleski2018automated} have been applied to automatically assign spectra. In particular, \cite{yeh2019automated} provide an algorithm for context-free assignment of rotational spectra, which does not require prior guesses of rotational constants.

Given parent and isotopologue rotational constants, we can convert them to planar moments of inertia $P_X > P_Y > P_Z$, which are the eigenvalues of the planar dyadic
(Appendix \ref{appendix:inertia}):
\begin{equation}\label{eq:moments}
\begin{gathered}
 I_X = \frac{h}{8\pi^2 A}, \quad I_Y = \frac{h}{8\pi^2 B}, \quad I_Z = \frac{h}{8\pi^2 C}, \\
 P_\zeta = \frac{1}{2}(I_X + I_Y + I_Z) - I_ \zeta, \quad \text{for } \zeta \in \{X, Y, Z\},
\end{gathered}
\end{equation}
where $h$ is Planck's constant.
These moments are effective moments because measured rotational constants are averaged over the ground vibrational state, but their values are sufficiently accurate for our purposes.

Now, consider an isotopologue where a single atom of the parent molecule has been isotopically substituted. This substitution perturbs the parent molecule's mass $M$ by some $m_\Delta$ and produces a new triplet of planar moments $P_X^\ast, P_Y^\ast, P_Z^\ast$. Kraitchman's equations \citep{kraitchman1953determination} relate the parent and isotopologue moments to the coordinates of the substituted atom $(x, y, z)$, but only up to a sign:
\begin{equation}\label{eq:kraitchmaneqn}
    \begin{gathered}
    |x| = \sqrt{\left(\cfrac{M + m_\Delta}{M m_\Delta}\right) \frac{(P_X^\ast - P_X)(P_Y^\ast - P_X)(P_Z^\ast - P_X)}{(P_Y - P_X)(P_Z - P_X)}},\\
    \vphantom{\sum} \text{ with } |y| \text{ and } |z| \text{ by cyclic permutation of } X, Y, Z.
    \end{gathered}
\end{equation}
These equations are derived from how isotopic substitution predictably shifts and rotates the principal axes of a molecule (Figure \ref{fig:large-overview}, \emph{top-left}).
Since effective moments were used in Kraitchman's equations, these substitution coordinates are susceptible to errors arising from zero-point vibrational effects, so that they are systematically smaller in magnitude than the true equilibrium coordinates \citep{kraitchman1953determination}. For coordinates lying near a principal axis, these errors can be so severe that \Eqref{eq:kraitchmaneqn} produces an imaginary value. Accordingly, any unavailable substitution coordinates, either due to low natural isotopic abundance or imaginary values, are set to null values.

Finally, isotopologue rotational constants need to be matched to the type of atom that was isotopically substituted (e.g., determine whether a set of rotational constants comes from a \textsuperscript{13}C or \textsuperscript{18}O isotopologue). In the context-free setting, this matching can be determined by examining line intensities of the isotopologues and comparing to natural abundance ratios. Alternatively, if the molecular formula is not too complex, one can match isotopologues to atom types by brute force.

In summary, the available information that will be provided are \textbf{(1)} a molecular formula, \textbf{(2)} a possibly incomplete set of unsigned substitution coordinates, and \textbf{(3)} the principal planar moments of inertia of the parent molecule. In turn, we wish to obtain a ranked list of all-atom 3D structures predicted to be the true equilibrium structure.

\newpage

\textbf{Related work.} 
\cite{mayer2019feasibility} searches through heavy atom frameworks consistent with a set of substitution coordinates, and is emulated by our genetic algorithm baseline.
\cite{mccarthy2020molecule} present a probabilistic machine learning framework for determining structures from a small set of spectroscopic parameters.
\cite{yeh2021progress} propose two methods for resolving sign ambiguities from Kraitchman's substitution coordinates using measurements of doubly-substituted isotopologues, electric dipole moments, and magnetic g-factors. 
However, as all these methods are proof-of-concept, they do not demonstrate success on large datasets.

\vspace*{\fill}

\begin{figure}[H]
\begin{center}
\includegraphics[width=0.9\textwidth]{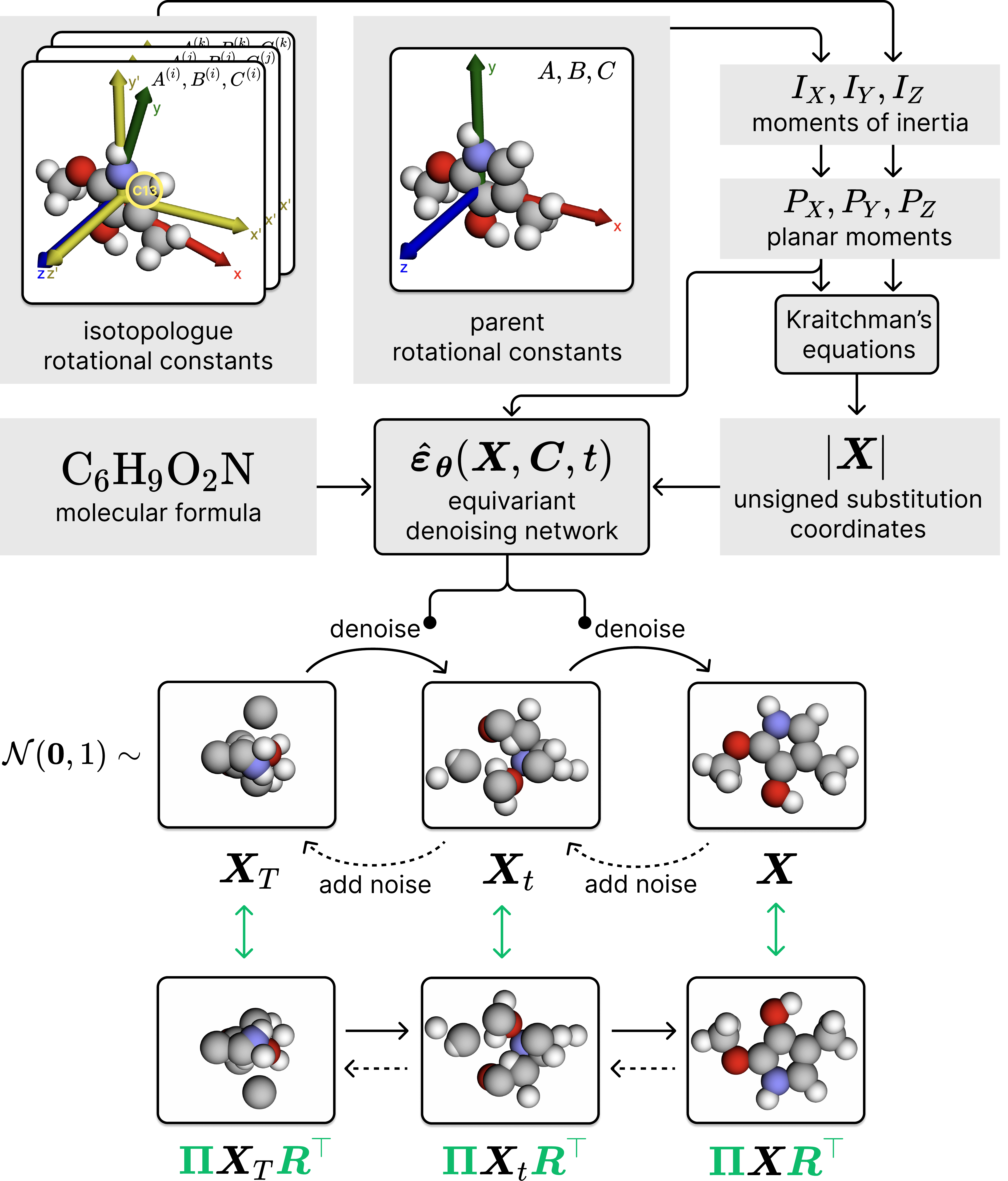}
\caption{\emph{Top:} Isotopologues have systematically shifted and rotated principal axes (yellow) relative to the parent's principal axes (RGB). The effect of isotopic substitution is magnified by 50$\times$ for visualization. Parent and isotopologue rotational constants are converted to planar moments of inertia and then to unsigned substitution coordinates using Kraitchman's equations. Given molecular formula, substitution coordinates, and planar moments of inertia of the parent molecule, \name{} learns to denoise random point clouds into all-atom 3D structures of the molecule. \emph{Bottom:} The denoiser model $\hat{\vepsilon}_\vtheta$ is equivariant with respect to axially-aligned reflections $\mR$ and node permutations $\mPi$, so that the modelled distribution $p_\vtheta(\mX|\mC)$ under the diffusion model is invariant to such transformations.}
\label{fig:large-overview}
\end{center}
\end{figure}

\vspace*{\fill}
\newpage 

\section{Approach}

The inputs of molecular formula, substitution coordinates, and moments of inertia specify a set of incomplete constraints on the true equilibrium structure of a molecule, as many unphysical point clouds could also satisfy these conditions within a certain tolerance.
However, molecules obey several steric and electronic rules known to chemists: atoms must have satisfactory valency, bonds must be of adequate length, and bond angles must minimize strain.
We leverage advances in diffusion models \citep{sohl2015deep, ho2020denoising} to learn and incorporate these constraints from data.
Indeed, diffusion models have already demonstrated success on several 3D molecular tasks, including conformer search \citep{xu2022geodiff, jing2022torsional}, docking \citep{corso2022diffdock}, and unconditional generation \citep{hoogeboom2022equivariant, schneuing2022structure, vignac2023midi, xu2023geometric}.
The basic idea behind diffusion models is that the true data distribution $p_{\text{data}}(\mX|\mC)$ can be transformed into a standard Gaussian distribution by progressively adding Gaussian noise across multiple timesteps $t$.
By training a neural network $\hat{\vepsilon}_\vtheta$ to predict the reverse of these diffusion steps, random noise can be iteratively denoised into samples from the original data distribution.
We recall the standard formulation of a diffusion model in Appendix \ref{appendix:diffusion}. 

For an unknown $N$-atom molecule, let $\mC$ be a node feature matrix encoding the prior information that is available. Concretely, we concatenate the following features of the molecule: 
\begin{equation}
\mC = 
\bigg(
\,\va \;\bigg|\; \vm \;\bigg|\; |\mX| \;\bigg|\; \mS \;\bigg|\; P_X, P_Y, P_Z\,
\bigg) \in \R^{N \times 11},
\end{equation}
where $\va$ are its atomic numbers, $\vm$ are its atomic masses, $|\mX|$ are its unsigned substitution coordinates (or $0$ where not given), $\mS$ is a binary mask that indicates which elements of $|\mX|$ are missing, and $P_X, P_Y, P_Z$ are its planar moments of inertia for the parent isotopologue (repeated along each row).
Then we aim to learn the true distribution of 3D conformations $\mX \in \R^{N \times 3}$ conditioned on $\mC$ with a diffusion model $p_\vtheta(\mX|\mC)$.

\subsection{Diffusion in the Principal Axis System}

Molecules in 3D contain geometric symmetries under translations, rotations, and reordering of atoms, which 
our diffusion model should account for to ensure good generalization \citep{elesedy2021provably}. 
A natural approach is to consider the principal axis system of the molecule, which is the coordinate system whose origin is the mass-weighted center-of-mass (CoM) and whose axes are the principal axes of rotation, or the orthonormal eigenvectors of the inertia matrix (Appendix \ref{appendix:inertia}).
This choice is motivated by the substitution coordinates being given in the principal axis system.
In particular, fixing the origin to the CoM removes translational symmetries \citep{xu2022geodiff}, while aligning the molecule to its principal axes reduces symmetries under rotoreflections to axially-aligned reflections across the $xy,yz,xz$ planes.
These reflectional symmetries persist since if $\vv$ is an eigenvector of the inertia matrix then so is $-\vv$, so the principal axes are only unique up to sign flips.
We then model a reflection-invariant distribution using a reflection-equivariant denoiser (Figure \ref{fig:large-overview}, \emph{bottom}).

\subsubsection{Zero CoM Subspace}

Fixing the CoM to the origin amounts to performing the diffusion process in the $3(N-1)$-dimensional linear subspace
\begin{equation}
    \sU = \left\{\mX = (\vx_1, \ldots, \vx_N)^\top \in \R^{N \times 3} \;\middle|\; \frac{1}{\sum_{j = 1}^N m_j}\sum_{i = 1}^N m_i\vx_i = \bm{0} \right\}. 
\end{equation}
Unlike $\R^{N \times 3}$, points in $\sU$ cannot be superimposed by translations. To diffuse over $\sU$, we require two minor changes from a regular diffusion model on $\R^{N \times 3}$ (Appendix \ref{appendix:diffusion}): \textbf{(1)} whenever a sample is drawn from a Gaussian distribution, it must be orthogonally projected onto $\sU$, and \textbf{(2)} the output of the network $\hat{\vepsilon}_\vtheta$ is restricted to $\sU$. This approach is similar to previous works which fix the \textit{unweighted} CoM at $\bm{0}$ \citep{xu2022geodiff, hoogeboom2022equivariant}, with the added difference that an orthogonal projection onto $\sU$ is not equivalent to simply removing the weighted CoM. Appendix Algorithms \ref{alg:train_step} and \ref{alg:sampling} summarize the modified procedures for training and sampling. Appendix \ref{appendix:Zero CoM Subspaces} formalizes and proves these claims.

\subsubsection{Axially-aligned Reflection Invariance}
Once the CoM has been fixed to the origin, diffusion takes place in the principal axis system of the true molecule by construction.
The remaining symmetries are stated as follows. For all axially-aligned reflections $\mR \in \{\diag(\vb) \mid \vb \in \{-1, +1\}^3 \}$ and node permutations $\mPi$ that map $\sU$ onto itself,
we desire that:
\begin{equation}
    p_\vtheta(\mPi\mX\mR^\top|\mPi\mC) = p_{\vtheta}(\mX|\mC).
\end{equation}
That is, the modelled distribution is invariant to axially-aligned reflections of $\mX$ and simultaneous reorderings of $\mX$ and $\mC$. This is satisfied if the denoiser network $\hat{\vepsilon}_\vtheta$ is correspondingly equivariant (Appendix \ref{appendix:invariance_and_equivariance}):
\begin{equation}
    \hat{\vepsilon}_\vtheta(\mPi\mX\mR^\top, \mPi\mC, t) = \mPi\hat{\vepsilon}_\vtheta(\mX, \mC, t)\mR^\top.
    \label{eq:network_equivariance}
\end{equation}
To ensure this, we implement $\hat{\vepsilon}_\vtheta$ with an architecture inspired from E($n$)-equivariant graph neural networks (EGNNs) \citep{satorras2021en}.
In the $n = 3$ case, EGNNs are designed to be equivariant under the E(3) group, which is the group of rigid 3D motions consisting of translations, reflections, rotations, and combinations thereof.
By modifying the EGNN with edge features that are reflection- but not E(3)-invariant, we relax its E(3)- to reflection-equivariance. 
We also found that incorporating Transformer-like elements \citep{vaswani2017attention} improved training dynamics and stability.
Further architectural details are given in Appendix \ref{appendix:architecture}. 

Frame averaging \citep{puny2021frame, duval2023faenet} arises as an unexpected connection to our approach.
This method ensures E(3)-equivariance by averaging over the frame of the E(3) group, where the frame consists of the $2^3$ possible point clouds that are aligned to the principal components of the point clouds and have the same \emph{unweighted} CoM.
However, instead of aligning to principal components, we align to principal axes of rotation and fix the \emph{weighted} CoM to be zero.

\subsection{Training and Inference}

During training, we process all examples so that they sit in their principal axis system. Simultaneously, we can easily compute the moments of inertia of the molecule and compute unsigned substitution coordinates by discarding signs (Appendix Algorithm \ref{alg:get_unsigned_coords}).
We only use substitution coordinates of atoms with naturally abundant isotopes, which for our datasets includes B, C, N, O, Si, S, Cl, Br, Hg and excludes H, F, Al, P, As, I, Bi.
However, not all of these substitution coordinates are available in practice due to low natural isotopic abundance or zero-point vibrational effects.
To mimic this effect, we further apply a random dropout with probability $p$ on the substitution coordinates, with $p$ itself being uniformly sampled from an interval $[p_{\textrm{min}}, p_{\textrm{max}}] \subseteq [0, 1]$ for each training example.
At inference time, we sample $K$ predictions for each example and rank predictions by deviation of their unsigned coordinates to the original unsigned substitution coordinates.
Additional details on training and sampling, including an explicit inference workflow, are available in Appendix \ref{appendix:training}.

\section{Experiments}

\textbf{Baseline.} To determine the effectiveness of search-based methods for solving this problem, we develop a genetic algorithm that searches over heavy atom frameworks by finding binary $+/-$ signs for heavy atoms with specified unsigned substitution coordinates and continuous positions for heavy atoms without.
The search is guided by a fitness function which measures agreement with zero CoM and the provided moments, combined with a likelihood term that is simply a pair distribution function of heavy atoms in the training set, inspired by \cite{mayer2019feasibility}.
Hydrogens are added afterwards using Hydride \citep{kunzmann2022adding}.
More details are available in Appendix \ref{appendix:baseline}.

\textbf{Datasets.} QM9 is an enumeration of 134k single-conformer molecules containing C, H, O, N, F with up to 9 heavy atoms calculated at the B3LYP/6-31G(2df,p) level of theory.
GEOM is a dataset of 292k drug-like molecules with conformers calculated by CREST \citep{pracht2020automated} at a semiempirical extended tight-binding level GFN2-xTB \citep{bannwarth2019gfn2}.
For GEOM, we use only the 30 lowest-energy conformers for each molecule, totalling 6.9M conformers.
Each dataset is then partitioned into training, validation, and test sets using an 80:10:10\% random split by molecule.

To provide a more realistic benchmark, we also drop substitution coordinates of all non-carbon atoms, followed by dropping 10\% of carbon substitution coordinates.
This emulates the fact that non-carbon isotopologues are sometimes difficult to detect in a spectrum due to low but non-zero abundance.
For example, the isotopic abundances of nitrogen and oxygen are respectively 2.5$\times$ and 5$\times$ lower than that for carbon.
We label these tasks as QM9-C and GEOM-C, both of which drop approximately 34\% of the original naturally-abundant isotopic substitution coordinates.

One model was trained on QM9 and tested on the QM9 and QM9-C tasks, and another model was trained on GEOM and tested on the GEOM and GEOM-C tasks.
For the test set of GEOM, only predictions for the lowest-energy conformer were evaluated, as the lowest-energy conformer will have the highest population in experiments.
For each test example, $K=10$ samples were generated and ranked by deviation from unsigned substitution coordinates.

\textbf{Metrics.} We evaluate generated samples in terms of connectivity correctness and all-atom RMSD.
A prediction is connectivity correct if RDKit's xyz2mol {\small\texttt{rdkit.Chem.rdDetermineBonds}} \citep{kim2015universal} returns the same SMILES string for both the prediction and the ground truth.
Connectivity correctness implies that the prediction and ground truth have the same bond connectivity, but may not necessarily be the same enantiomer or diastereomer.
We additionally measure each prediction's all-atom RMSD after alignment to the ground truth via minimum RMSD over all $2^3$ possible reflections.
Since our diffusion model is permutation invariant, node orderings are not preserved, and we must find an assignment of each atom in the ground truth to each atom in the predicted sample.
To do so, for each reflection, we solve a linear assignment problem \citep{crouse2016implementing} to match atoms that are close in space and have the same atom type but may not have the same node index.

\begin{table}
  \centering
  \setlength{\tabcolsep}{10pt}
  \begin{tabular}{c l c c c c}
    \toprule
    & & \multicolumn{2}{c}{\textbf{Correctness (\%)}}  & \multicolumn{2}{c}{\textbf{Median RMSD (\AA)}}  \\
    \textbf{Method} & \textbf{Task} & $k = 1$ & $k = 5$ & $k = 1$ & $k = 5$ \\
    \midrule %
    & QM9 & 7.33 & 12.4 & 0.842 & 0.546  \\
    Genetic & QM9-C & 0.127 & 0.225 & 1.29 & 1.05 \\
    algorithm & GEOM & 0.0308 & 0.0377 & 2.05 & 1.77 \\
    & GEOM-C & 0.00342 & 0.00685 & 2.31 & 2.08 \\
    \midrule %
    \multirow{4}{*}{\name{}} & QM9 & 99.9 & 99.9 & 0.00626 & 0.00625 \\
     & QM9-C & 91.3 & 93.1 & 0.00935 & 0.00882 \\
     & GEOM & 98.9 & 99.2 & 0.00617 & 0.00616 \\
    & GEOM-C & 32.6 & 35.8 & 1.17 & 0.765 \\
    \bottomrule
  \end{tabular}
  \caption{The genetic algorithm and \name{} are benchmarked on the test sets of QM9 and GEOM. Performance is measured by connectivity correctness and median all-atom RMSD.}
  \label{table:results}
\end{table}
\begin{figure}[t]
\begin{center}
\includegraphics[width=\textwidth]{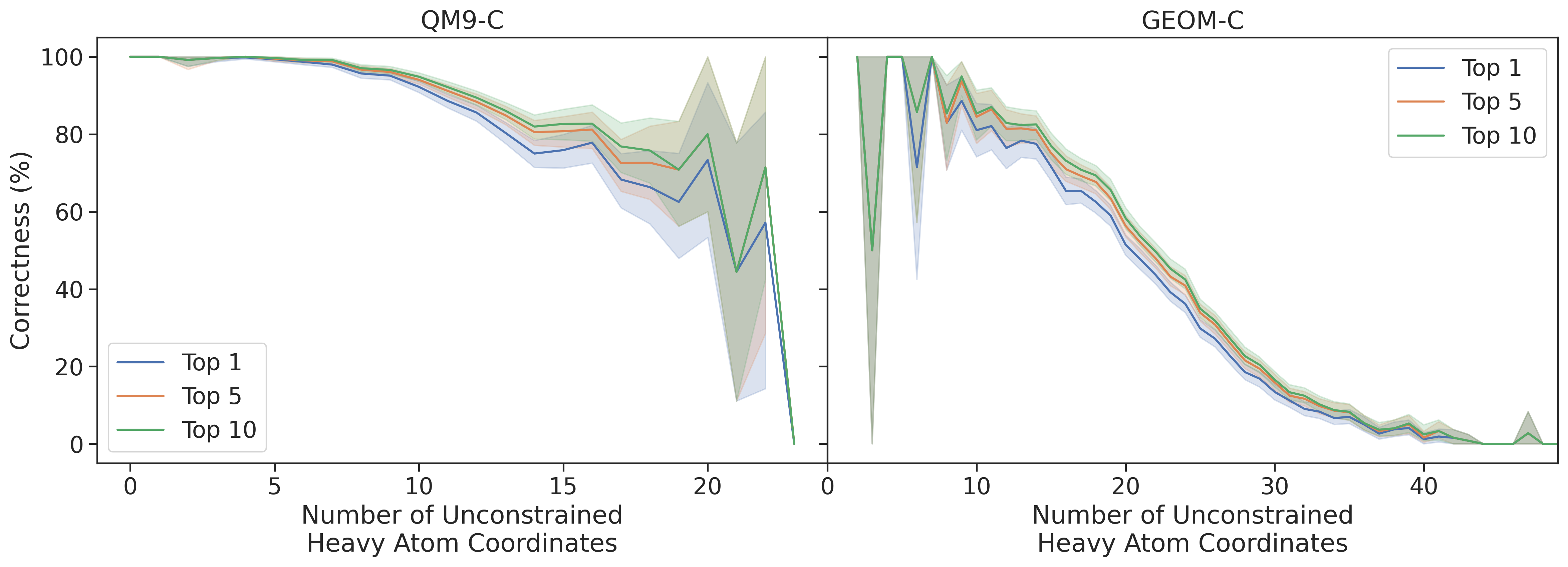}
\caption{Average connectivity correctness of \name{} on QM9-C and GEOM-C for samples with different numbers of unconstrained heavy atom coordinates. Shading shows the 95\% confidence interval. The fewer the number of substitution coordinates that are provided, the greater the number of unconstrained heavy atom coordinates, and the greater the difficulty of the task, as shown by decreases in average correctness.}
\label{fig:accuracy}
\end{center}
\end{figure}

\textbf{Accuracy.} Table \ref{table:results} indicates that \name{} can predict structures with near perfect accuracy on both QM9 and GEOM when provided with all naturally abundant substitution coordinates, even without hydrogen, fluorine, and phosphorus.
When restricting the provided substitution coordinates, the model retains 91.3\% top-1 accuracy for QM9-C, while performance drops to 32.6\% top-1 accuracy for GEOM-C.
In comparison, performance by the genetic algorithm is poor.
While this simple baseline demonstrates appreciable results for QM9, even achieving a top-10 heavy-atom-correctness of 46.4\% (Appendix Table \ref{table:extra_results}), it is not able to predict structures from GEOM.
Low baseline performance is likely due to inefficient generation of continuous coordinates and a poorly adapted fitness function for larger molecules.
Additional top-$k$ metrics and visualized top-1 predictions for all methods and tasks are shown in Appendix \ref{appendix:results}.

The significant drop in performance when moving from GEOM to GEOM-C is in line with the fact that GEOM contains molecules with significantly more atoms than QM9: QM9 contains an average of 8.8 heavy atoms per molecule, while GEOM contains an average of 24.8 heavy atoms per molecule.
In Figure \ref{fig:accuracy}, we see that the difficulty of an example is correlated with the number of unconstrained heavy atom coordinates, which is the number of heavy atom coordinates for which substitution coordinates are \emph{not} provided.
Indeed, further experiments found that even if provided with \emph{only} molecular formula and moments, \name{} can still obtain reasonable accuracy on QM9 (48.8\% top-100 accuracy), owing to the small size of molecules in QM9.
We discuss these results in Appendix \ref{appendix:no_sub}.

\begin{figure}
\begin{center}
\includegraphics[width=\textwidth]{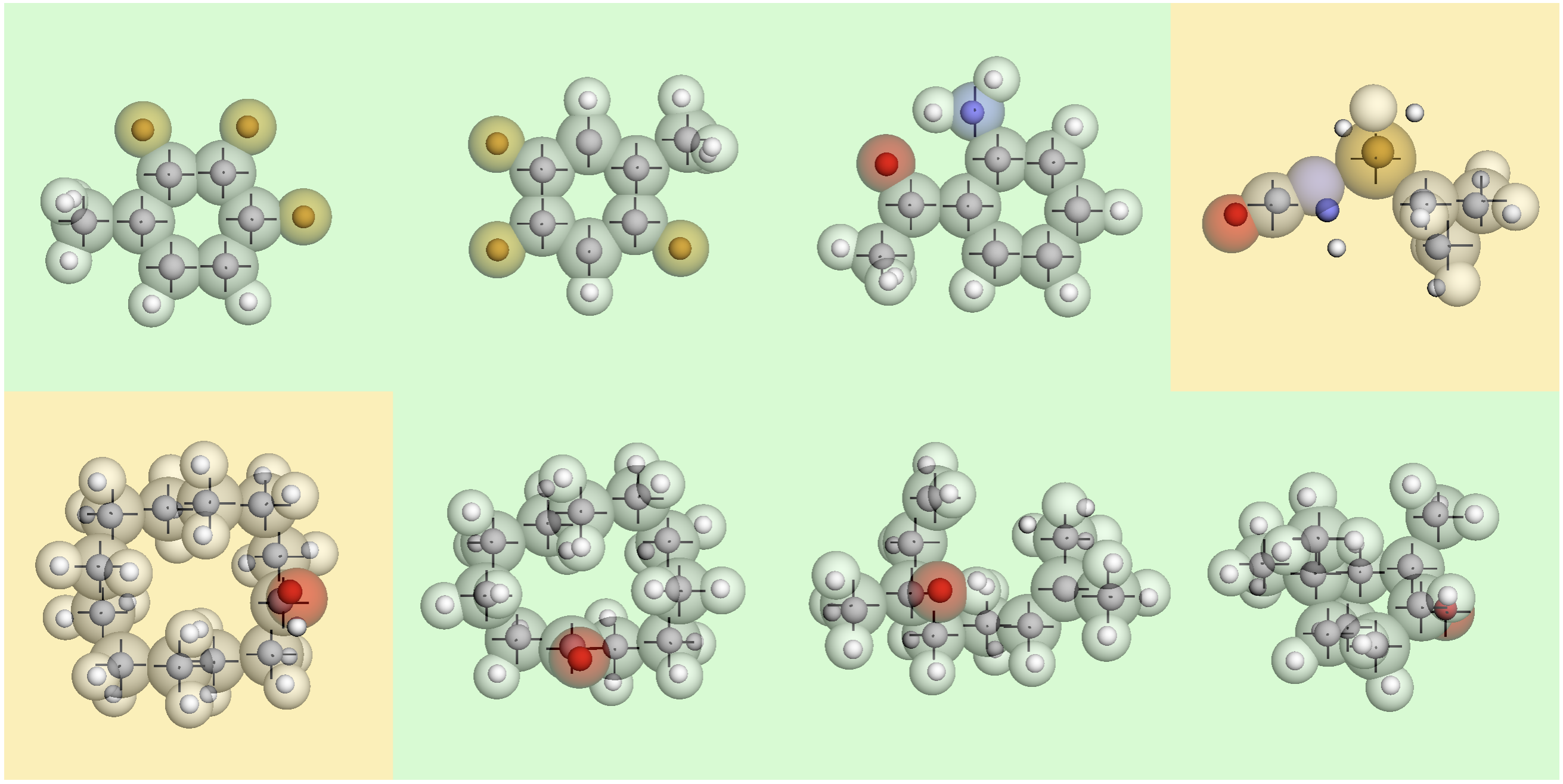}
\caption{Top-1 predictions for experimental substitution coordinates of molecules which do not appear in QM9 or GEOM. The transparent structure is the ground truth, while the small spheres indicate the top-1 predicted structure. Green indicates all-atom correctness while yellow indicates heavy atom correctness. The presence of black pins indicates whether a substitution coordinate in that direction was available.}
\label{fig:ood-samples}
\end{center}
\end{figure}

\subsection{Experimentally Measured Substitution Coordinates}

To benchmark \name{}'s applicability to experimental data, we extracted from the literature a small dataset of 33 conformers with experimentally measured substitution coordinates.
Experimental measurements are subject to zero-point vibrational effects, which cause small inaccuracies in substitution coordinates.
For these experiments, we generated $K=100$ samples for each test example.
For the 8 examples which do not appear in QM9 or GEOM, our top-1 predictions are shown in Figure \ref{fig:ood-samples}.
\name{} predicts correct heavy atom frameworks for all examples and correct all-atom structures for 6 of 8 examples, demonstrating strong performance even on examples that are both out-of-dataset and subject to zero-point vibrational effects.

Of the remaining molecules, 7 were very similar to examples in QM9 or GEOM, while 18 did not have ground truth Cartesian coordinates provided, so they were manually verified by visual comparison to pictures in the original paper.
In total, 25 of the 33 examples were all-atom correct, while 29 of the 33 examples were correct up to hydrogens.
The robustness of \name{} to inaccuracies due to zero-point vibrational effects may be attributed to its loose conditioning on substitution coordinates: since the model is not forced to have exact agreement with the substitution coordinates, it can account for errors in them.
Top-1 predictions are visualized in Appendix \ref{appendix:lit_results} along with references.

\section{Conclusion}

We introduce \name{}, a reflection-equivariant diffusion model for inferring all-atom 3D structure from molecular formula, Kraitchman's substitution coordinates, and moments of inertia.
We validate our approach on large datasets and obtain high accuracy, especially as more substitution coordinates are provided.
Additionally, we find that \name{} is applicable to experimentally measured substitution coordinates, demonstrating its potential for context-free 3D structure determination.

Our approach relies on being able to determine many of the substitution coordinates accurately.
The main obstacle to this is being able to detect and assign rotational constants to enough isotopologues, even for atoms with low isotopic abundance such as oxygen and nitrogen.
We must also know what type of atom was substituted in each isotopologue, which may be nontrivial to resolve.
Furthermore, our method inherits the constraints of Kraitchman analysis and rotational spectroscopy, such as the rigid rotor approximation, requirement of a permanent dipole moment, and inability to distinguish enantiomers.
Lastly, we require a second instrument to determine molecular formula.

However, we believe that many of these limitations can be addressed in future work.
One promising direction is to explore more powerful methods for conditioning the diffusion model, such as guidance \citep{dhariwal2021diffusion,ho2022classifier}.
Given that analytic expressions for moments of inertia are available, it would also be interesting if the diffusion process could be constrained so that the moments of inertia are exactly preserved by the end of the diffusion process.
This would remove the need for substitution coordinates and enable structure determination solely from moments and molecular formula.
One might be able to sidestep these input requirements altogether by conditioning directly on rotational spectra.
An orthogonal direction could be to evaluate reflection-equivariant diffusion in the principal axis system as a general method for modelling 3D point cloud data.

\subsubsection*{Acknowledgments}
We thank Naruki Yoshikawa, Luca Thiede, and Kin Long Kelvin Lee for helpful discussions, and Andy Cai for help with visualizations.
Resources used in preparing this research were provided by Intel, the Province of Ontario, the Government of Canada through CIFAR, companies sponsoring the Vector Institute ({\small \url{www.vectorinstitute.ai/#partners}}), Calcul Québec, and the Digital Research Alliance of Canada ({\small \url{alliancecan.ca}}).
This research was undertaken thanks in part to funding provided to the University of Toronto’s Acceleration Consortium from the Canada First Research Excellence Fund CFREF-2022-00042.
We acknowledge the Defense Advanced Research Projects Agency (DARPA) under the Accelerated
Molecular Discovery Program under Cooperative Agreement No.\ HR00-\linebreak{}11920027 dated August 1, 2019.
A.A.-G. thanks Anders G.~Frøseth for his generous support. A.A.-G also acknowledges the generous support of Natural Resources Canada and the Canada 150 Research Chairs program.

We acknowledge the Python community \citep{van1995python,oliphant2007python} for developing the core set of tools that enabled this work, including PyTorch \citep{paszke2019pytorch}, PyTorch Lightning \citep{Falcon_PyTorch_Lightning_2019}, DGL \citep{wang2019dgl}, RDKit \citep{greg_landrum_2023_8413907}, py3Dmol \citep{rego20153dmol}, Jupyter \citep{kluyver2016jupyter},
Matplotlib \citep{hunter2007matplotlib}, seaborn \citep{waskom2021seaborn}, NumPy \citep{2020NumPy-Array}, SciPy \citep{2020SciPy-NMeth}, and pandas \citep{The_pandas_development_team_pandas-dev_pandas_Pandas}.

\bibliography{refs}
\bibliographystyle{iclr2023_conference}

\newpage

\appendix
\section{Moments of inertia}
\label{appendix:inertia}

For an $N$-atom molecule, let $m_i$ and $\vx_i \in \R^3$ be the mass and position of the $i$-th atom. Assume also that the coordinates have been centered to have zero CoM, i.e., $\sum_{i = 1}^N m_i\vx_i = \mathbf{0}$. Then the inertia matrix is given by
\begin{equation}\label{eq:inertia}
    \inertiadyadic = \sum_{i = 1}^N m_i\left(\lVert\vx_i\rVert_2^2 \mI_3 - \vx_i \vx_i^\top \right) \in \R^{3 \times 3},
\end{equation}
where $\mI_3$ is the identity matrix. The eigenvalues $I_X, I_Y, I_Z$ of $\inertiadyadic$ are the principal moments of inertia, while its orthonormal eigenvectors form the principal axes of rotation.
A closely related matrix is the planar dyadic
\begin{equation}
    \planardyadic = \sum_{i = 1}^N m_i \vx_i \vx_i^\top \in \R^{3 \times 3},
\label{eq:planar_dyadic}
\end{equation}
which contains the same information but leads to simpler expressions for Kraitchman's equations.
The planar dyadic shares the same eigenvectors as the inertia matrix, and its eigenvalues $P_X, P_Y, P_Z$ are the planar moments of inertia, which are related to $I_X, I_Y, I_Z$ by Equation \ref{eq:moments}.

\section{Denoising Diffusion}
\label{appendix:diffusion}

 At a high level, diffusion models learn to iteratively \textit{denoise} samples drawn from an elementary prior distribution into samples from the desired distribution. Given a real data sample $\vx \in \R^m$, let $\vz_0, \ldots, \vz_T$ be a sequence of increasingly noisier random variables drawn from the isotropic Gaussian distributions
\begin{equation}\label{q_zt_given_xt}
    q(\vz_t|\vx) = \mathcal{N}(\vz_t; \alpha_t\vx, \sigma_t^2),
\end{equation}
where $\alpha_t \in (0, 1)$ are chosen to be decreasing with $\alpha_0 \approx 1$ and $\alpha_T \approx 0$, and $\sigma_t^2 = 1 - \alpha_t^2$. We also require $\alpha_T$ to be sufficiently small such that $\vz_T$ is nearly indistinguishable from random noise, i.e., $q(\vz_T|\vx) \approx \mathcal{N}(\vz_T; \bm{0}, 1)$. Now, we can equivalently model the variables through a Markovian noising process with isotropic Gaussian transitions 
\begin{equation}
    q(\vz_t|\vz_{t - 1}) = \mathcal{N}\!\left(\vz_t; \tilde{\alpha}_t\vz_{t - 1}, \tilde{\sigma}_t^2\right),
\end{equation}
where $\tilde{\alpha}_t = \alpha_t/\alpha_{t - 1}$ and $\tilde{\sigma}_t^2 = \sigma_t^2 - \tilde{\alpha}_t^2 \sigma_{t-1}^2$, and $\alpha_{-1} = 1$ and $\vz_{-1} = \vx$ for the case $t = 0$. In fact, the transition posteriors conditioned on $\vx$ are again isotropic and Gaussian: for $t \geq 1$,
\begin{equation}\label{transition posterior}
    q(\vz_{t - 1}|\vx, \vz_t) =  \mathcal{N}\!\left(\vz_{t - 1};
    \frac{\tilde{\alpha}_t \sigma_{t - 1}^2}{\sigma_t^2}\vz_{t} + \frac{\alpha_{t - 1}\tilde{\sigma}_t^2}{\sigma_t^2}\vx, \frac{\tilde{\sigma}_t^2\sigma_{t - 1}^2}{\sigma_t^2}\right).
\end{equation}
\Eqref{transition posterior} specifies a true denoising process in which we can transform a heavily corrupted sample $\vz_T \sim q(\vz_T|\vx)$ back into its original state $\vx$. 

However, $\vx$ is not known during inference, so instead, diffusion models approximate it with a neural network $\hat{\vx}_\vtheta(\vz_t, \vc, t) \approx \vx$, where $\vtheta$ is the network parameters and $\vc$ denotes any given conditioning information. In practice, it is more common to indirectly parameterize $\hat{\vx}_\vtheta$ via:
\begin{equation}
    \hat{\vx}_\vtheta(\vz_t, \vc, t) = \frac{1}{\alpha_t}\vz_t - \frac{\sigma_t}{\alpha_t}\hat{\vepsilon}_\vtheta(\vz_t, \vc, t),
\end{equation}
which is motivated by reparameterizing \Eqref{q_zt_given_xt} as $\vz_t = \alpha_t\vx + \sigma_t\vepsilon$ for $\vepsilon \sim \mathcal{N}(\vepsilon; \bm{0}, 1)$, and then solving for $\vx$. Informally, this $\vepsilon$-parameterization learns to predict the unscaled noise that was added to a corrupted sample, while the former $\vx$-parameterization learns to directly predict the denoised example itself. The marginal generative distribution under the model is then:
\begin{equation}
     p_\vtheta(\vx|\vc) = p_\vtheta(\vx|\vz_0, \vc)\left( \prod_{t=1}^T p_\vtheta(\vz_{t-1}|\vz_t, \vc) \right)p(\vz_T),
\end{equation}
\begin{align}
    p_\vtheta(\vx|\vz_0, \vc) &= \mathcal{N}\!\left(\vx; \hat{\vx}_\vtheta(\vz_0, \vc, 0), \alpha_0^2\sigma_0^{-2}\right), \label{ptheta_x_given_z0} \\
    p_\vtheta(\vz_{t-1}|\vz_t, \vc) &= q(\vz_{t - 1}|\hat{\vx}_\vtheta(\vz_t, \vc, t), \vz_t), \\
    p(\vz_T | \vc) &= \mathcal{N}(\vz_T; \bm{0}, 1),
\end{align}
which can be sampled from sequentially.\footnote{We follow \cite{hoogeboom2022equivariant} in \Eqref{ptheta_x_given_z0}. } The model is trained by minimizing a squared error loss that is simplified from a variational lower bound on the log-likelihood $\log p_\vtheta(\vx|\vc)$:
\begin{equation}
    \E_{\mathcal{U}(t; 0, \ldots, T),\, \mathcal{N}(\vepsilon;\mathbf{0}, 1)}\left[\lVert \vepsilon - \hat{\vepsilon}_\vtheta(\alpha_t\vx + \sigma_t\vepsilon, \vc, t)\lVert_2^2 \right].
\end{equation}
For derivations of the preceding claims, we refer readers to \cite{ho2020denoising}.

\section{Derivations}

\subsection{Zero CoM Subspaces}\label{appendix:Zero CoM Subspaces}

\newcommand{\subspace}{\sU}
\newcommand{\proj}{\mPhi}

Herein, we will treat point clouds $\mX \in \R^{N \times 3}$ as vectors $\rvx = \vectorize(\mX) \in \R^{3N}$ by concatenating their columns. Assuming \emph{normalized} point masses $\tilde{\vm} \in (0, 1)^N$ that sum to 1, the CoM of $\rvx$ is:
\begin{equation}
    \vx_{\text{CoM}}  = (\mI_3 \otimes \tilde{\vm}^\top)\rvx \in \R^3,
\end{equation}
where $\mI_{3}$ is the identity matrix and $\otimes$ is the Kronecker product. The set of zero CoM point clouds 
\begin{equation}
    \subspace = \{\rvx \in \R^{3N} \mid  \vx_{\text{CoM}} = \mathbf{0} \}
\end{equation}
is an $m$-dimensional linear subspace, where $m = 3(N - 1)$, so there is an isometric isomorphism $\varphi \colon  \R^m \to \subspace$. 
Using $\varphi$, we can establish a diffusion model over $\subspace$ by pushing forward a diffusion model over $\R^m$. In fact, due to structure-preserving properties of $\varphi$, we can train and sample over $\subspace$ without actually realizing $\varphi$ or interacting with the underlying space $\R^m$.

Specifically, training and sampling from a standard diffusion on $\R^m$ (Appendix \ref{appendix:diffusion}) requires only a few types of operations in $\R^m$: \textbf{(1)} taking linear combinations of points, \textbf{(2)} computing the distance between two points, and \textbf{(3)} sampling from isotropic Gaussian distributions. The first two operations can equivalently be done in $\sU$. Formally, if $\vu_i \in \R^m$ and $\rvx_i = \varphi(\vu_i)$ for $1 \leq i \leq k$, then 
\begin{equation}
    \varphi\left(\sum_{i = 1}^k \alpha_i\vu_i\right) = \sum_{i = 1}^k \alpha_i\rvx_i \quad\text{and}\quad \lVert\vu_1 - \vu_2 \rVert_2 = \lVert \rvx_1 - \rvx_2 \rVert_2
\end{equation}
for all $\alpha_i \in \R$, since $\varphi$ is an isomorphism and isometry, respectively. The corresponding operation for \textbf{(3)} is more involved and first requires some propositions.

\begin{proposition}\label{appendix:AphiAphiT}
    Let $\mA_\varphi \in \R^{3N \times m}$ be the matrix of $\varphi$. Then $\mA_\varphi \mA_\varphi^\top \in \R^{3N \times 3N}$ is the matrix $\proj$ for the orthogonal projection of $\R^{3N}$ onto $\subspace$.
\end{proposition}
\begin{proof}
    Since $\varphi$ is an isomorphism, the columns of $\mA_\varphi$ form a basis of $\subspace$. A classical result from linear algebra states that the orthogonal projection matrix can be obtained by 
    \begin{equation}
        \proj = \mA_\varphi(\mA_\varphi^\top\mA_\varphi)^{-1} \mA_\varphi^\top 
        = \mA_\varphi \mA_\varphi^\top.
    \end{equation}
    The final equality follows since $\varphi$ is an isometry, so $\mA_\varphi^\top\mA_\varphi = \mI_m$ is the identity.
\end{proof}
\begin{proposition}\label{appendix:orthoproj}
    Let $\proj\rvx$ be the orthogonal projection of $\rvx \in \R^{3N}$ onto $\subspace$. Then    
    \begin{equation}
        \proj \rvx = \rvx - \frac{1}{||\tilde{\vm}||_2^2}\vectorize(\tilde{\vm}\vx_{\normalfont \text{CoM}}^\top).
    \end{equation}
\end{proposition}
\begin{proof}
    Evaluating $\proj\rvx$ is equivalent to solving the following problem over $\rvp \in \R^{3N}$:
    \begin{equation}
        \text{minimize: } ||\rvp - \rvx||_2^2, \qquad \text{subject to: } \vp_{\text{CoM}} = \mathbf{0},
    \end{equation}
    This is a quadratic minimization problem with linear equality constraints, whose solution $\rvp_\star$ satisfies
    \begin{equation}\label{appendix:prop2 eqn1}
        \begin{pmatrix}
            \mI_{3N} & \mI_3 \otimes \tilde{\vm} \\
            \mI_3 \otimes \tilde{\vm}^\top & \mathbf{0}
        \end{pmatrix}
        \begin{pmatrix}
            \rvp_\star \\ 
            \vlambda
        \end{pmatrix}
         = 
         \begin{pmatrix}
             \rvx \\
             \mathbf{0}
         \end{pmatrix}
    \end{equation}
    for some Lagrange multipliers $\vlambda \in \R^3$. Expanding \Eqref{appendix:prop2 eqn1}, we have
    \begin{align}
        \rvp_\star &= \rvx - (\mI_3 \otimes \tilde{\vm})\vlambda, \label{appendix:prop2 eqn2a}\\ 
        (\mI_3 \otimes \tilde{\vm}^\top)\rvp_\star &= \textbf{0}. \label{appendix:prop2 eqn2b}
    \end{align}
    Substituting \Eqref{appendix:prop2 eqn2a} into \Eqref{appendix:prop2 eqn2b} gives
    \begin{equation}
        \textbf{0} = (\mI_3 \otimes \tilde{\vm}^\top)(\rvx - (\mI_3 \otimes \tilde{\vm})\vlambda) = \vx_{\text{CoM}} - ||\tilde{\vm}||_2^2 \vlambda,
    \end{equation}
    and solving for $\vlambda$ and substituting it back into \Eqref{appendix:prop2 eqn2a} gives
    \begin{equation}
    \rvp_\star = \rvx - \frac{1}{||\tilde{\vm}||_2^2}(\mI_3 \otimes \tilde{\vm})\vx_{\text{CoM}} = \rvx - \frac{1}{||\tilde{\vm}||_2^2}\vectorize(\tilde{\vm}\vx_{\text{CoM}}^\top),
    \end{equation}
    as desired. 
\end{proof}

Now, consider sampling $\vu \sim \mathcal{N}(\bar{\vu}, \sigma^2)$ from an isotropic Gaussian in $\R^m$. Let $\bar{\rvx} = \varphi(\bar{\vu})$. Then the mapping of $\vu$ onto $\subspace$ follows the distribution
\begin{equation}
    \varphi(\vu) \sim \mathcal{N}(\varphi(\bar{\vu}), \sigma^2\mA_\varphi \mA_\varphi^\top) = \mathcal{N}(\bar{\rvx}, \sigma^2 \proj) 
\end{equation}
by Proposition \ref{appendix:AphiAphiT}. Hence, we can equivalently orthogonally project a sample $\rvx \sim \mathcal{N}(\bar{\rvx}, \sigma^2)$ from an isotropic Gaussian in $\R^{3N}$ since
\begin{equation}
    \proj\rvx \sim \mathcal{N}(\proj\bar{\rvx}, \sigma^2\proj\proj^\top) = \mathcal{N}(\bar{\rvx}, \sigma^2\proj).
\end{equation}
Proposition~\ref{appendix:orthoproj} explains how to actually compute $\proj\vx$. Indeed, in the uniform-mass case (unweighted CoM) $\tilde{\vm} = \frac{1}{N}\mathbf{1}$, where $\mathbf{1} \in \R^N$ is the ones-vector, Proposition \ref{appendix:orthoproj} simplifies to 
\begin{equation}
\proj \rvx = \rvx - \vectorize(\mathbf{1}\vx_{\text{CoM}}^\top).
\end{equation}
That is, $\proj\rvx$ is computed by subtracting the CoM from each point in $\rvx$. However, note that this is not true of the general nonuniform-mass case. 

One final and minor technicality is that if our diffusion process occurs in the underlying space $\R^m$, our denoiser network would be a function $\hat{\vepsilon}_\vtheta(\cdot, \vc, t) \colon \R^m \to \R^m$. The equivalent network on the subspace $\sU$ is a function $\hat{\vepsilon}^\varphi_\vtheta(\cdot, \vc, t) \colon \sU \to \sU$ given by
\begin{equation}\label{eq:imageofdenoiser}
\hat{\vepsilon}^{\varphi}_\vtheta(\rvx, \vc, t) = \varphi(\hat{\vepsilon}_\vtheta(\varphi^{-1}(\rvx), \vc, t)).  
\end{equation}
However, since we want to work solely in $\sU$, instead of directly implementing $\hat{\vepsilon}_\vtheta$, we implement $\hat{\vepsilon}^{\varphi}_\vtheta$ which implicitly determines  $\hat{\vepsilon}_\vtheta$ through \Eqref{eq:imageofdenoiser}.

\subsection{Invariance and Equivariance}\label{appendix:invariance_and_equivariance}

Let us continue the notation and discussion from Appendix \ref{appendix:Zero CoM Subspaces}, except we will now treat point clouds as $N \times 3$ matrices and also assume the conditioning features $\mC$ are given as $N \times d$ matrices.  
 
Let $\sT \subseteq \mathrm{O}(3) = \{\mR \in \R^{3 \times 3} \mid \mR^{-1} = \mR^\top\}$ be a set of rigid 3D linear transformations of interest. For example, previous works \citep{hoogeboom2022equivariant, xu2022geodiff} take $\sT = \mathrm{O}(3)$ to be the space of all such motions, while we consider only the subset $\sT = \{\diag(\vb) \mid \vb \in \{-1, +1\}^3 \}$ of axially-aligned reflections.  
Let $\sP$ be the set of permutation matrices $\mPi \in \{0, 1\}^{N \times N}$ mapping $\sU$ onto itself, so that $\mPi\mX \in \sU$ for all $\mX \in \sU$. Then we will call distributions
\begin{equation}
    p(\cdot\,|\,\cdot) \colon \R^m \times \R^{N \times d} \to \R \quad\text{and}\quad p(\cdot\,|\,\cdot, \cdot)\colon \R^m \times \R^m \times \R^{N \times d} \to \R
\end{equation}
$(\sT, \sP)$-invariant and $(\sT, \sP)$-equivariant, respectively, if for each $\mR \in \sT$ and $\mPi \in \sP$, the following hold for all $\vu, \vu' \in \R^m$ and $\mC \in \R^{N \times d}$:
\begin{equation}
    p(\lambda(\vu) | \mPi \mC) = p(\vu | \mC) \quad\text{and}\quad  p(\lambda(\vu) | \lambda(\vu'), \mPi \mC) = p(\vu | \vu', \mC),
\end{equation}
where $\lambda(\vu) = \varphi^{-1}(\mPi\varphi(\vu)\mT^\top)$. The preceding statement is saying that the likelihood of a point clouds in $\sU$ should be unchanged regardless of how it is transformed or reordered under $(\sT, \sP)$, but is formalized as a condition on the underlying space $\R^m$.  
In the following propositions, we prove that our diffusion model learns a $(\sT, \sP)$-invariant generative distribution if its denoiser network satisfies a corresponding equivariance condition.

\begin{proposition}\label{appendix:lambda isometry}
    Any such $\lambda$ defined above is an isometric isomorphism, so its matrix $\mA_\lambda \in \R^{m\times m}$ is orthogonal with $|\det\mA_\lambda| = 1$.
\end{proposition}
\begin{proof}
    Linearity follows from the linearity of $\varphi$ and $\varphi^{-1}$. Since both are also isometries,
    \begin{equation}
        ||\lambda(\vu)||_2^2 = \lVert\mPi\varphi(\vu)\mR^\top\rVert_F^2 = \lVert\varphi(\vu)\mR^\top\rVert_F^2 =||\varphi(\vu)||_F^2 = ||\vu||_2^2,
    \end{equation}
    where $||\cdot||_F$ is the Frobenius norm. The second equality follows since permuting the entries of a matrix does not change its norm. The third follows since applying a rigid motion to each row of a matrix does not change each row's norm and hence does not change the overall matrix norm.   
\end{proof}

\begin{proposition}\label{appendix:gaussian invariance}
    The standard Gaussian distribution $p(\vz_T | \mC) = \mathcal{N}(\vz_T; \mathbf{0}, 1)$ is $(\sT, \sP)$-invariant.
\end{proposition}
\begin{proof}   
    This follows from Proposition \ref{appendix:lambda isometry} and the fact that a Gaussian distribution centered at $\mathbf{0}$ depends on $\vz_T$ only through its norm. 
\end{proof}

\begin{proposition}
    If a diffusion model has $(\sT, \sP)$-equivariant transitions $p_{\vtheta}(\vz_t|\vz_{t+1}, \mC)$ and $p_{\vtheta}(\vx|\vz_0)$, then its marginal generative distribution $p_{\vtheta}(\vx|\mC)$ will be $(\sT, \sP)$-invariant. 
\end{proposition}
\begin{proof}
    Let any $\lambda$ be fixed as above. Then 
    \begin{equation}
        p_{\vtheta}(\lambda(\vz_{T-1})|\mPi\mC) = \int_{\R^{m}} p_\vtheta(\lambda(\vz_{T-1})|\vz_T, \mPi\mC)p(\vz_T| \mPi\mC) \mathrm{d}\vz_T.
    \end{equation}
    Recall that the transitions are $(\sT, \sP)$-equivariant and the marginal distribution of $\vz_T$ (which is chosen to be the standard Gaussian) is $(\sT, \sP)$-invariant by  Proposition \ref{appendix:gaussian invariance}. Hence,
    \begin{equation}
        p_{\vtheta}(\lambda(\vz_{T-1})|\mPi\mC) = \int_{\R^m} p_\vtheta(\vz_{T-1}|\lambda^{-1}(\vz_T), \mC)p(\lambda^{-1}(\vz_T) | \mC) \mathrm{d}\vz_T.
    \end{equation}
    Finally, a change of variables $\vu = \lambda^{-1}(\vz_T)$ gives
    \begin{align}
        p_{\vtheta}(\lambda(\vz_{T-1})|\mPi\mC) &= \int_{\R^m} p_\vtheta(\vz_{T-1}|\vu, \mC)p(\vu | \mC)|\det\mA_\lambda| \mathrm{d}\vu \\
         &= \int_{\R^m} p_\vtheta(\vz_{T-1}|\vu, \mC)p(\vu | \mC) \mathrm{d}\vu \qquad \text{(Proposition \ref{appendix:lambda isometry})} \\
          &= p_\vtheta(\vz_{T-1}| \mC),
    \end{align}
    so the marginal distribution of $\vz_{T - 1}$ is $(\sT, \sP)$-invariant. We can repeat this argument inductively to find that the marginal distributions of $\vz_{T - 2}, \ldots, \vz_0$, and $\vx$, are $(\sT, \sP)$-invariant.
\end{proof}

\begin{proposition}\label{appendix:equivariancecondition}
Suppose the denoiser network satisfies, for each $\mR \in \sT$ and $\mPi \in \sP$,
\begin{equation}\label{appendix:denoiserequivariance}
\hat{\vepsilon}_\vtheta^\varphi(\mPi\mX\mR^\top, \mPi\mC, t) = \mPi\hat{\vepsilon}_\vtheta^\varphi(\mX, \mC, t)\mR^\top
\end{equation}
for all $\mX \in \sU$ and $\mC \in \R^{N \times d}$. Then the transitions of the diffusion model are $(\sT, \sP)$-equivariant.
\end{proposition}
\begin{proof}
By \Eqref{eq:imageofdenoiser}, the denoiser on the underlying space is $\hat{\vepsilon}_\vtheta(\vu, \mC, t) = \varphi^{-1}(\hat{\vepsilon}_\vtheta^{\varphi}(\varphi(\vu), \mC, t))$. Direct computation shows that \Eqref{appendix:denoiserequivariance} is equivalent to requiring: 
\begin{equation}\label{appendix:underlyingdenoiserequivariance}
\hat{\vepsilon}_\vtheta(\lambda(\vu), \mPi\mC, t) = \lambda(\hat{\vepsilon}_\vtheta(\vu, \mC, t)).
\end{equation}
Recall Proposition \ref{appendix:lambda isometry}. It follows that: 
\allowdisplaybreaks
\begin{align}
    p_\vtheta(\lambda(\vx)|\lambda(\vz_0), \vc) &= \mathcal{N}\!\left(\lambda(\vx); \hat{\vx}_\vtheta(\lambda(\vz_0), \vc, 0), \alpha_0^2\sigma_0^{-2}\right) \\
    &= \mathcal{N}\!\left(\lambda(\vx); \lambda(\hat{\vx}_\vtheta(\vz_0, \vc, 0)), \alpha_0^2\sigma_0^{-2}\right) \\
    &= \mathcal{N}\!\left(\vx; \hat{\vx}_\vtheta(\vz_0, \vc, 0), \alpha_0^2\sigma_0^{-2}\right) \\
    &= p_\vtheta(\vx|\vz_0, \vc),
\end{align}
so the transition from $\vz_0$ to $\vx$ is $(\sT, \sP)$-equivariant. The second equality follows from \Eqref{appendix:underlyingdenoiserequivariance} and linearity of $\lambda$. The third equality follows as a normal distribution $\mathcal{N}(\vx;\bm{\mu}, \sigma^2)$ is $(\sT, \sP)$-equivariant since it depends only on the distance $\lVert \vx - \bm{\mu}\rVert_2$ and $\lambda$ is an isometry. The proof of the 
equivariance of $p_{\vtheta}(\vz_t|\vz_{t+1}, \mC)$ follows through similar manipulations.  
\end{proof}
\section{Baseline Genetic Algorithm}
\label{appendix:baseline}

Inspired by \cite{mayer2019feasibility}, our genetic algorithm (GA) searches over heavy atom frameworks that are consistent with a given set of unsigned substitution coordinates. Hydrogens are added later. For an example with $k$ specified heavy atom unsigned substitution coordinates and $u$ unavailable heavy atom substitution coordinates, each individual in the genetic population comprises a vector of binary signs $\vb \in \{+, -\}^k$ and continuous positions $\vu \in \R^u$ which, together, fully characterize a candidate heavy atom framework.  

\begin{figure}[h]
    \centering
    \includegraphics[width=\textwidth]{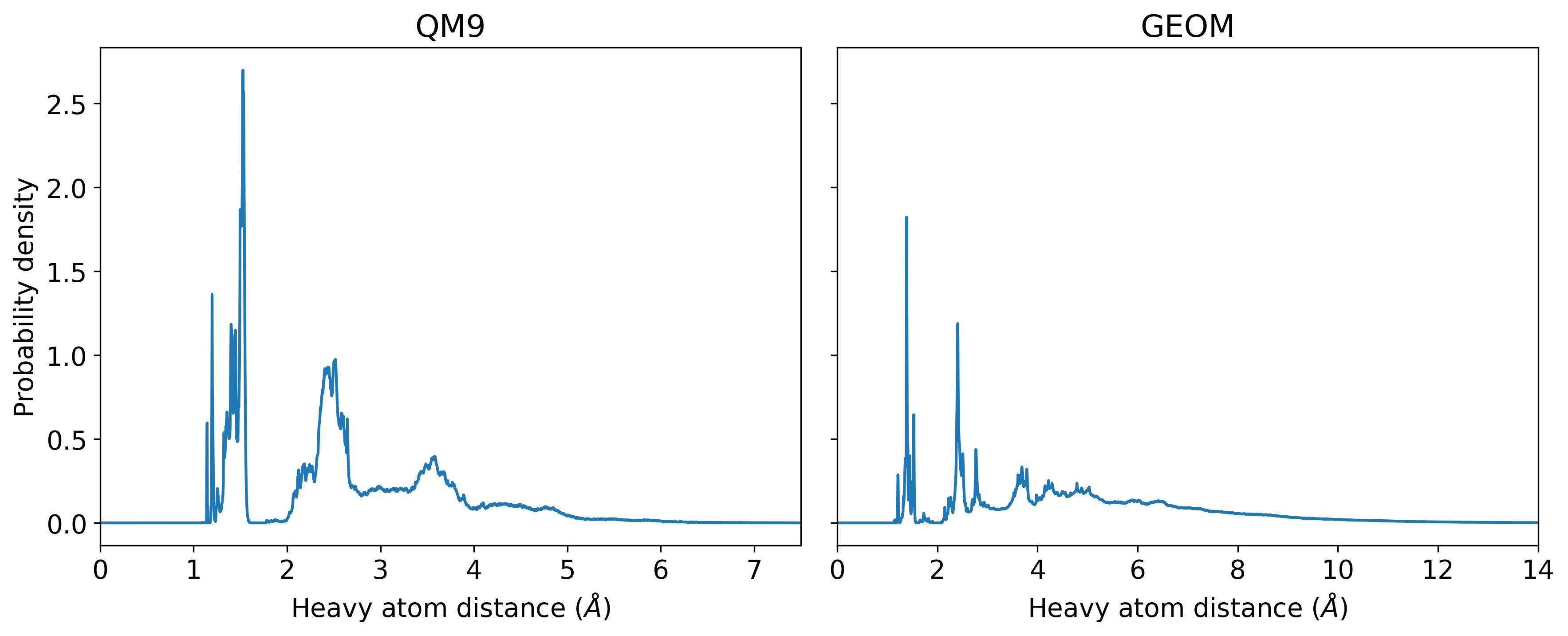}
    \vspace{-\baselineskip}
    \caption{Pairwise heavy atom distance histograms used for the GA fitness function.}
    \label{fig:distogram}
\end{figure}

Our GA comprises a fitness, mutation, crossover, and selection function. The fitness function assigns to each candidate framework a scalar ``badness" score by taking a sum of three terms:
\begin{enumerate}
    \item \textbf{Moment error.} The Euclidean distance between the upper-triangular parts of the planar dyadic of the candidate (\Eqref{eq:planar_dyadic}) and the ground truth $\diag(P_X, P_Y, P_Z)$.  
    \item \textbf{CoM error.} The Euclidean distance of the candidate's CoM from the origin.
    \item \textbf{Pairwise distance NLL.} A probability density function $\tilde{p}(d)$ of pairwise Euclidean distances between heavy atoms is estimated as a histogram from training set statistics (Figure \ref{fig:distogram}). We consider the negative log-likelihood $\text{NLL}(\sD) = -\sum_{d\in\sD} \log{\tilde{p}(d)}$, where $\sD$ is the multiset of pairwise distances of atoms in a candidate framework, assuming they are independent and identically distributed (i.i.d.). Candidates with a lower NLL more closely match the training set and are more likely to make up actual molecules.
\end{enumerate}
The mutation operator takes an individual and randomly bit-flips one or more elements in $\vb$ and adds Gaussian noise to $\vu$. The crossover operator takes in two individuals and randomly swaps selected contiguous swaths of signs between each $\vb$ and does not modify $\vu$. Lastly, the selection operator uses a tournament selection. Given these components, our GA minimizes the fitness function using the \code{eaSimple} algorithm from the Python library DEAP \citep{DEAP_JMLR2012} with 20 generations, a population size of 20,000, a mutation rate of 0.7, and a crossover rate of 0.9. At the end of the GA, the top $K$ scoring heavy atom frameworks with unique signs $\vb$ are kept. 

Next, each framework is decorated with hydrogens using Hydride \citep{kunzmann2022adding}, which predicts hydrogen positions by comparing heavy atom fragments to a library of hydrogen-containing fragments. Since Hydride may not necessarily add the correct number of hydrogens consistent with the molecular formula, hydrogens are randomly dropped if there are too many, and hydrogens are added with positions as Gaussian noise if there are too few. To mitigate the effects of this randomness, the adding or dropping procedure is repeated 1000 times. These hydrogen-decorated structures are scored using another function that sums two terms: \textbf{(1)} the distance of the CoM to the origin, and \textbf{(2)} $1000\min(0, 1.09-d_{\textrm{min}})^2$, 
where $d_{\textrm{min}}$ is the minimum pairwise distance between any two atoms in the structure. This second term provides a large penalty if atoms are too close together. Only the highest scoring hydrogen-decorated structure for each heavy atom framework is kept. In total, $K$ all-atom 3D structures are returned. On QM9 and GEOM, the GA and decoration take approximately 2 minutes per example.

\section{Model and Training Details}
\label{appendix:model_details}

\subsection{Pseudocode}

\begin{algorithm}[h]
\caption{Computing the training objective on a single example.}
\label{alg:train_step}
\begin{algorithmic}[1]
    \Require{The conditioning features $\mC \in \R^{N \times d}$ and 3D conformation $\mX \in \R^{N \times 3}$ of a molecule; a neural network $\hat{\vepsilon}_\theta$.}
    \State Orient $\mX$ to its principal axis system using Algorithm \ref{alg:get_unsigned_coords}
    \State $t \sim \mathcal{U}(t; 0, \ldots, T)$ 
    \State $\vepsilon \sim \mathcal{N}(\vepsilon; \mathbf{0}, 1)$ with dimensions $\vepsilon \in \R^{N \times 3}$
    \State Orthogonally project $\vepsilon$ onto the zero CoM subspace $\sU$
    \State $\mZ_t \gets \alpha_t\mX + \sigma_t\vepsilon$ 
    \State \Return $\lVert \vepsilon -  \hat{\vepsilon}_\theta(\mZ_t, \mC, t) \rVert_2^2$
\end{algorithmic}
\end{algorithm}

\begin{algorithm}[h]
\caption{Conditionally sampling a single example.}
\label{alg:sampling}
\begin{algorithmic}[1]
    \Require{The conditioning features $\mC \in \R^{N \times d}$ of an unknown molecule; a neural network $\hat{\vepsilon}_\theta$.}
    \State $\mZ_T \sim \mathcal{N}(\vepsilon; \mathbf{0}, 1)$ with dimensions $\mZ_T \in \R^{N \times 3}$
    \State Orthogonally project $\mZ_T$ onto the zero CoM subspace $\sU$
    \For{$t = T$ to $1$}
        \State  $\mZ_{t-1} \sim p_\vtheta(\mZ_{t - 1}|\mZ_t, \mC)$
        \State Orthogonally project $\mZ_t$ onto $\sU$
    \EndFor
    \State \Return $\mX \sim p_\vtheta(\mX|\mZ_0, \mC)$
\end{algorithmic}
\end{algorithm}

\begin{algorithm}[h]
\caption{Preprocessing a single example.}
\label{alg:get_unsigned_coords}
\begin{algorithmic}[1]
    \Require{The atomic masses $\vm \in (0, \infty)^{N}$ and 3D conformation $\mX \in \R^{N \times 3}$ of a molecule.}
    \State $M \gets \sum_{i = 1}^N m_i$
    \State $\vx_{\text{CoM}} \gets M^{-1} \sum_{i = 1}^N m_i \vx_i$, where $\vx_i$ is the $i$-th row of $\mX$
    \State $\mX \gets (\vx_1 - \vx_{\text{CoM}}, \ldots, \vx_N - \vx_{\text{CoM}})^\top$
    \State $\planardyadic \gets \sum_{i = 1}^N m_i \vx_{i} \vx_{i}^\top$
    \State Eigendecompose $\planardyadic = \mV\diag(P_X, P_Y, P_Z)\mV^\top$, where $P_X > P_Y > P_Z$
    \State $\mX \gets \mX\mV$ 
    \State $\mS \gets N \times 3$ binary dropout mask based on atom types and dropout rate
    \State $|\mX| \gets \mS \odot \textrm{abs}(\mX)$
    \State \Return (
    \State  $\quad$ aligned coordinates $\mX$, 
    \State  $\quad$ partial unsigned substitution coordinates $|\mX|$ with mask $\mS$, 
    \State  $\quad$ planar moments of inertia $(P_X, P_Y, P_Z)$
    \State )
\end{algorithmic}
\end{algorithm}

\begin{algorithm}[h]
\caption{Network architecture of $\hat{\vepsilon}_\vtheta^\varphi(\mX, \mC, t)$. We overload the notation for layers (e.g., every use of \texttt{Lin} denotes a new linear layer). See Appendix \ref{appendix:architecture} for further details. } 
\label{alg:architecture}
\begin{algorithmic}[1]
    \Require{The conditioning features $\mC \in \R^{N \times d}$ and noised 3D conformation $\mX \in \sU \subseteq \R^{N \times 3}$ of a molecule (centered to zero CoM); a timestep $t \in \mathbb{N}$; network parameters $\vtheta$.}
    \State $\sigma \gets \texttt{SiLU}$
    \State $\vt \gets \texttt{SinPosEmb}(t)$  \Comment{sinusoidal timestep embedding (128-dim)}
    \State $\mA \gets \texttt{Embed}(\va)$ \Comment{learned atomic number embedding (32-dim)}
    \State $\mC \gets (\,\mC \mid \mA \mid \tilde{\vm} \mid \vt\, )$ 
    \State $\mH \gets \texttt{Lin}(\mC)$ \Comment{linearly project to hidden states (256-dim)}
    \State $\mH_{\text{res}} \gets \mH$
    \State $\mC \gets \sigma(\texttt{Lin}(\sigma(\texttt{Lin}(\mC))))$ \Comment{MLP (dims $d_{\text{in}} \to 128 \to 128$)}
    \State $\mX^{(0)} \gets \mX$
    \For{$i$ from $1$ to $n_{\text{blocks}} - 1$}  \Comment{ResiDual-like blocks ($n_{\text{blocks}} = 6$)}
        \State $\mX', \mH' \gets \texttt{EQBlock}(\mX, \mH, \mX^{(0)})$
        \State $\mX \gets \mX'$ with its CoM subtracted 
        \State $\mH_{\text{res}} \gets \mH_{\text{res}} + \mH'$
        \State $\mH \gets \texttt{AdaLN}(\mH + \mH', \mC)$
        \State $\mH' \gets \texttt{Lin}(\sigma(\texttt{Lin}(\mH)))$ \Comment{MLP (dims $ 256 \to 320 \to 256$)}
        \State $\mH_{\text{res}} \gets \mH_{\text{res}} + \mH'$
        \State $\mH \gets \texttt{AdaLN}(\mH + \mH', \mC)$
        \If{$i = n_{\text{blocks}}  - 1$} 
            \State $\mH \gets \mH + \texttt{AdaLN}(\mH_{\text{res}}, \mC)$
        \EndIf
    \EndFor
    \State $\mX', \cdot \gets \texttt{EQBlock}(\mX, \mH, \mX^{(0)})$ \Comment{final coordinate update}
    \State $\mX \gets \mX'$ with its CoM subtracted 
    \State \Return $\mX - \mX^{(0)}$
\end{algorithmic}
\end{algorithm}

\newpage

\subsection{Network Architecture}
\label{appendix:architecture}

To satisfy Proposition \ref{appendix:equivariancecondition}, we implement the denoiser $\hat{\vepsilon}^\varphi_\vtheta(\mX, \mC, t)$ with an architecture inspired by Transformers \citep{vaswani2017attention} and E($n$)-equivariant graph neural networks (EGNNs) \citep{satorras2021en}. Algorithm~\ref{alg:architecture} summarizes the network architecture. The core body of the network is a ResiDual Transformer backbone \citep{xie2023residual}, except the self-attention layers are replaced with an equivariant block (\texttt{EQBlock}) that jointly updates the hidden features and coordinates, which we will describe shortly. We also use a conditional version of LayerNorm (\texttt{AdaLN}) \citep{dieleman2022continuous, dhariwal2021diffusion}, where the affine scale and shift are given from a linear projection of some input node features $\mC$, instead of being learned constants as in regular LayerNorm.

\texttt{EQBlock} is adapted from the EGNN block \citep{satorras2021en}, with the major difference being that we add reflection- but not E(3)-invariant edge features. Specifically, for $\vx, \vx' \in \R^3$, we consider the ``distance" features:
\begin{equation}
    \Delta(\vx, \vx') = \left( \lVert\vx - \vx'\rVert_2^2, \vx \cdot \vx', \lVert \vx \rVert_2^2, \lVert \vx' \rVert_2^2, s(\vx - \vx'), s(\vx), s(\vx') \right) \in \R^{13},
\end{equation}
where $s(\vx) = (x_1^2, x_2^2, x_3^2)$ denotes an element-wise squaring operation. The first four features above are E(3)-invariant, while the last three are only reflection-invariant. Features that are not translation-invariant, such as $\vx \cdot \vx'$, were used as suggested by \cite{vignac2023midi}. Then, we perform the following operations within $\texttt{EQBlock}(\mX, \mH, \mX^{(0)})$ to produce an updated point cloud $\mX'$:
\begin{align}
    \vm_{j \to i} &\gets \texttt{MLP}(\vh_i, \vh_j, \Delta(\vx_i, \vx_j), \Delta(\vx^{(0)}_i, \vx^{(0)}_j)), \quad \text{for all } i, j, \label{eq:messagemlp} \\
    \vx_i' &\gets \sum_{j = 1, \, j \neq i}^N \left(\frac{\vx_i - \vx_j}{||\vx_i - \vx_j||_2^2 + 1}\right) \odot \texttt{Lin}(\vm_{j \to i}), \quad \text{for all } i, \label{eq:coordupdate}
\end{align}
where $\vx_i$, $\vh$, $(\vx^{(0)}_i$, and $\vx'_i$ denote the $i$-th row of $\mX$, $\mH$, $\mX^{(0)}$, and $\mX'$; and $\odot$ denotes a Hadamard product. In the \Eqref{eq:messagemlp} MLP, we use two linear layers with dimensions $d_{\text{in}} \to 320 \to 320$ and SiLU activations for each one. In \Eqref{eq:coordupdate}, we use a linear layer that down-projects with dimensions $320 \to 3$. We also output node features $\mH'$ using an attention-like update:
\begin{align}
    a_{j \to i}, \vv_{j \to i} &\gets \texttt{Lin}(\vm_{j \to i}), \quad \text{for all } i, j, \\
    \vo_i &\gets \sum_{j = 1,\, j \neq i}^N \left(\frac{\exp(a_{j \to i})}{\sum_{k = 1,\, k \neq i}^N \exp(a_{k \to i})} \right) \vv_{j \to i}, \quad \text{for all } i, \\
    \vh_i' &\gets \texttt{Lin}(\vo_i), \quad \text{for all } i,
\end{align}
where $\vh'_i$ denotes the $i$-th row of $\mH'$. In practice, we extend the above equations to employ multiple smaller heads (8 heads, 32-dim each), analogous to multi-headed attention.

\subsection{Training and Inference Details}
\label{appendix:training}

We use the same diffusion noise schedule and number of diffusion steps ($T = 1000$) as \cite{hoogeboom2022equivariant}. Table \ref{tab:my_label} gives the training hyperparameters of \name{} on each dataset. To handle the larger molecules in GEOM, we conduct distributed training over 8 GPUs. On both datasets, we use the Adam optimizer \citep{kingma2015adam} with no weight decay; a linear learning rate warmup over 2000 steps \citep{ma2021on}; and an adaptive gradient clipping strategy from \cite{hoogeboom2022equivariant}, whereby we clip the gradient norm at $1.5\mu + 2\sigma$, where $\mu$ and $\sigma$ are the mean and standard deviation of the gradient norms of the 50 previous optimizer steps. For sampling, we use an exponential moving average (EMA) of the network parameters.

\begin{table}[h]
    \centering
    \begin{tabular}{l|cc}
        \toprule 
         \textbf{Hyperparameter} & \textbf{QM9} & \textbf{GEOM} \\
         \midrule 
         Number of GPUs & 1 & 8 \\ 
         GPU model (NVIDIA) &  Quadro RTX 6000 &  TITAN V\\ 
         GPU memory & 24 GB & 8 $\times$ 12 GB \\
         Epochs & 6000 & 200 \\
         Training time & 130 h & 126 h \\
         Training steps & 1.24M & 0.46M \\
         Effective batch size  & 512 & $240$ \\
         Coordinate dropout $[p_{\text{min}}, p_{\text{max}}]$ & $[0, 1]$  &  $[0, 0.5]$ \\
         Optimizer & Adam & Adam \\ 
         Learning rate & $4 \times 10^{-4}$ & $2 \times 10^{-4}$ \\
         Learning rate warmup steps & 2000 & 2000 \\
         Gradient clipping & Yes & Yes  \\
         EMA decay & 0.999 & 0.9995 \\
         \bottomrule
    \end{tabular}
    \caption{Training hyperparameters of \name{} 
 on QM9 and GEOM.}
    \label{tab:my_label}
\end{table}

A detailed workflow of predicting 3D structure at inference time is depicted in Figure \ref{fig:workflow}. Generating $33 \times 100$ samples for the predictions of the literature dataset required 14 minutes, or about 25.5 seconds per example, on a single NVIDIA RTX A6000 GPU with 48 GB RAM.

\begin{figure}[h]
\begin{center}
\includegraphics[width=\textwidth]{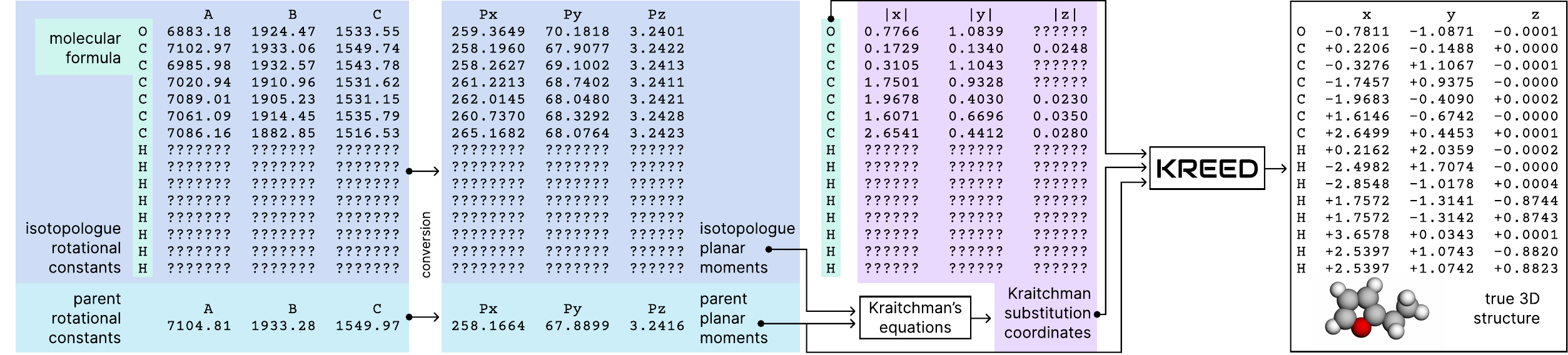}
\caption{Explicit workflow for all-atom 3D structure prediction given molecular formula and parent and isotopologue rotational constants.}
\label{fig:workflow}
\end{center}
\end{figure}

\newpage

\newpage

\section{Additional Results on QM9 and GEOM}
\label{appendix:results}

As seen in Table \ref{table:extra_results}, the genetic algorithm (GA) performs considerably better when only considering the correctness of the heavy atoms, indicating that the method of adding hydrogens is not satisfactory. However, the GA is still unable to deal with examples with many atoms or that have a significant number of missing substitution coordinates. This can be attributed to inefficient search of continuous coordinates, as well as the i.i.d.\ assumption of the pairwise distance NLL term breaking down for examples with many atoms.

\begin{table}[h]
  \centering
  \setlength{\tabcolsep}{7pt}
  \begin{tabular}{c l c c c c c c}
    \toprule
    & & \multicolumn{3}{c}{\textbf{Correctness (\%)}}  & \multicolumn{3}{c}{\textbf{Heavy Correctness (\%)}}  \\
    \textbf{Method} & \textbf{Task} & $k = 1$ & $k = 5$ & $k = 10$ & $k = 1$ & $k = 5$ & $k = 10$ \\
    \midrule %
    & QM9 & 7.33 & 12.4 & 15.0 & 25.4 & 39.6 & 46.4 \\
    Genetic & QM9-C & 0.127 & 0.225 & 0.337 & 0.262 & 0.584 & 0.914 \\
    algorithm & GEOM & 0.0308 & 0.0377 & 0.0377 & 0.158 & 0.253 & 0.318 \\
    & GEOM-C & 0.00342 & 0.00685 & 0.00685 & 0.00342 & 0.00685 & 0.0103 \\
    \midrule
    \multirow{4}{*}{\name{}}  & QM9 & 99.9 & 99.9 & 99.9 & 99.9 & 100. & 100. \\
     & QM9-C & 91.3 & 93.1 & 93.8 & 92.6 & 94.8 & 95.8 \\
     & GEOM & 98.9 & 99.2 & 99.2 & 99.4 & 99.5 & 99.5 \\
    & GEOM-C & 32.6 & 35.8 & 37.0 & 37.0 & 41.6 & 42.9 \\
    \bottomrule
  \end{tabular}
  \caption{Top-$k$ all-atom and heavy-atom connectivity correctness for both the GA and \name{}.}
  \label{table:extra_results}
\end{table}

\subsection{No Substitution Coordinates}
\label{appendix:no_sub}
Even if provided with \emph{only} molecular formula and moments, \name{} can still obtain reasonable accuracy on QM9.
$K=100$ samples were generated for each example, and generated predictions were ranked by deviation from the true moments.
Given that 10 times more samples were generated per example, these tasks were only evaluated for a small portion of each test set to minimize computational cost.
Top-1 all-atom connectivity correctness was 27.9\% for QM9 ($n=1335$), but is 0.273\% for GEOM ($n=1464$).
It was initially unexpected that accuracy on QM9 could be much greater than 0, given that this task provides very few input constraints.
However, since molecules in QM9 are much smaller, it makes sense that there are only a few stable molecules with a given molecular formula and set of moments of inertia.
In addition, we set $p_{\textrm{max}}=1$ when training on QM9, so that some examples seen during training had almost all substitution coordinates dropped, but $p_{\textrm{max}}=0.5$ when training on GEOM.
It is conceivable that accuracy on this task could be raised for GEOM by increasing $p_{\textrm{max}}$, however, we suffered training instabilities when doing so.

\begin{table}[H]
\centering
    \begin{tabular}{l c c c c c c c c}
    \toprule
    & \multicolumn{4}{c}{\textbf{Correctness (\%)}}  & \multicolumn{4}{c}{\textbf{Heavy Correctness (\%)}}  \\
    \textbf{Task} & $k = 1$ & $k = 5$ & $k = 10$ & $k = 100$ & $k = 1$ & $k = 5$ & $k = 10$ & $k = 100$\\
    \midrule
QM9 & 27.9 & 39.4 & 42.1 & 48.8 & 29.2 & 42.3 & 46.4 & 56.1 \\
GEOM & 0.273 & 0.342 & 0.342 & 0.546 & 0.273 & 0.410 & 0.478 & 0.820 \\
\bottomrule
    \end{tabular}
    \caption{Top-$k$ all-atom and heavy-atom connectivity correctness of \name{} when provided with \emph{no} substitution coordinates.}
    \label{table:no_sub}
\end{table}

\subsection{Visualizations}

In the following figures, a green background indicates all-atom connectivity correctness, a yellow background indicates heavy atom connectivity correctness, and a red background indicates incorrectness. The presence of black pins indicates whether a substitution coordinate in that direction was available.

\vspace{\fill}
\clearpage

\begin{figure}[hp]
\begin{center}
\includegraphics[width=\textwidth]{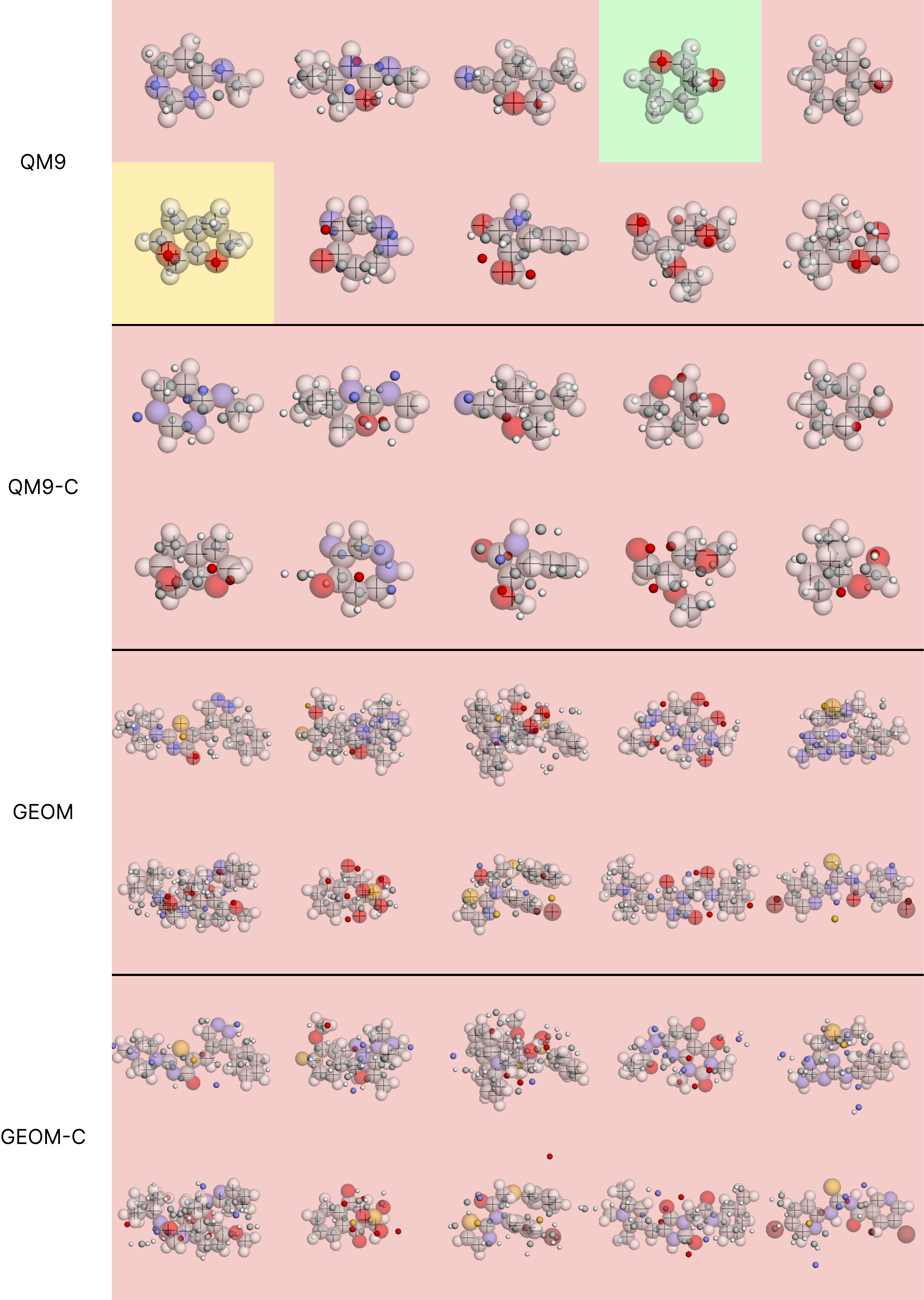}
\caption{Top-1 predictions of the GA on random test set examples.}
\label{fig:baseline-samples}
\end{center}
\end{figure}

\newpage

\begin{figure}[hp]
\begin{center}
\includegraphics[width=\textwidth]{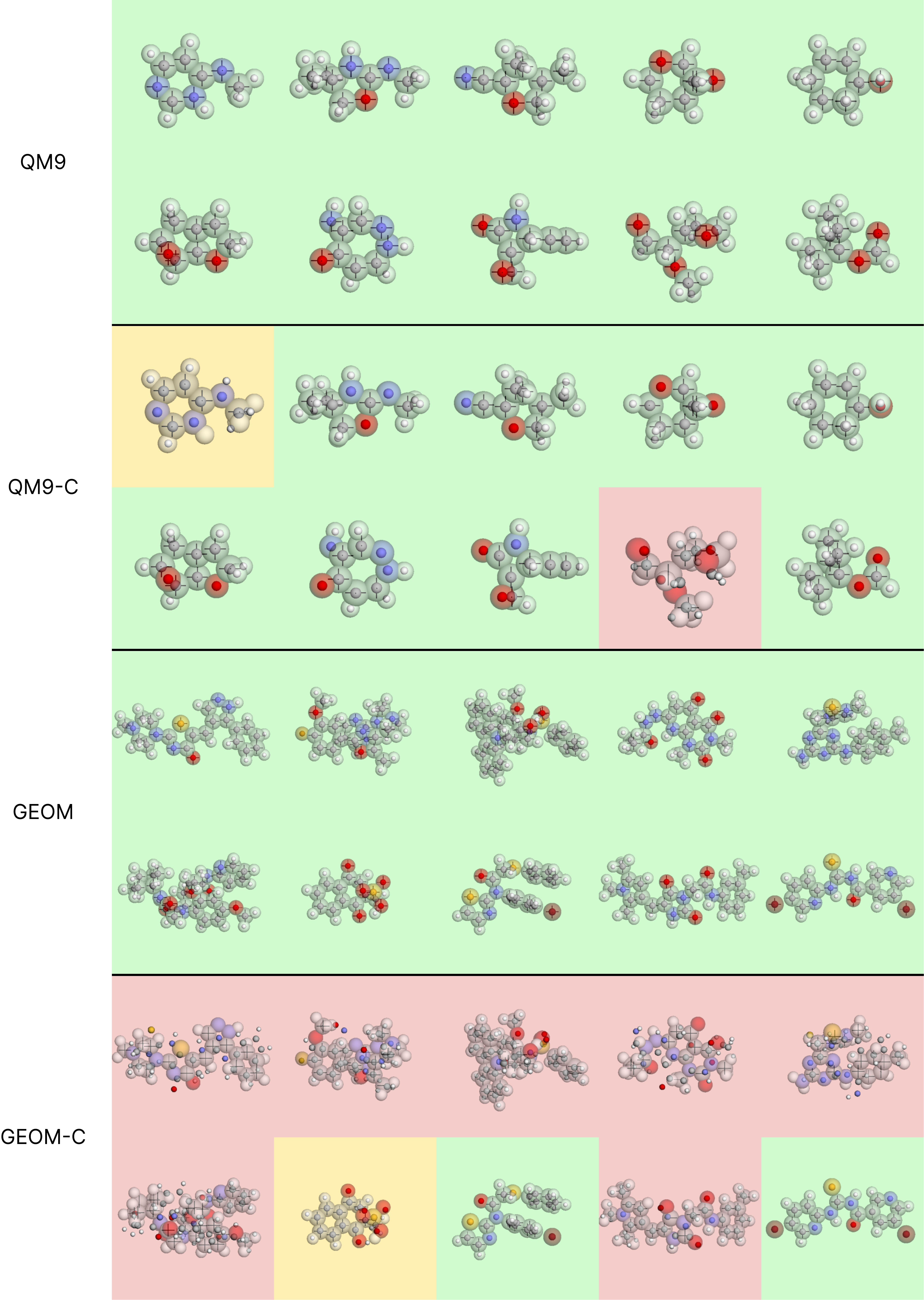}
\caption{Top-1 predictions of \name{} on random test set examples.}
\label{fig:ddpm-samples}
\end{center}
\end{figure}

\newpage

\section{Experimentally Measured Substitution Coordinates}
\label{appendix:lit_results}

\begin{figure}[H]
\begin{center}
\includegraphics[width=\textwidth]{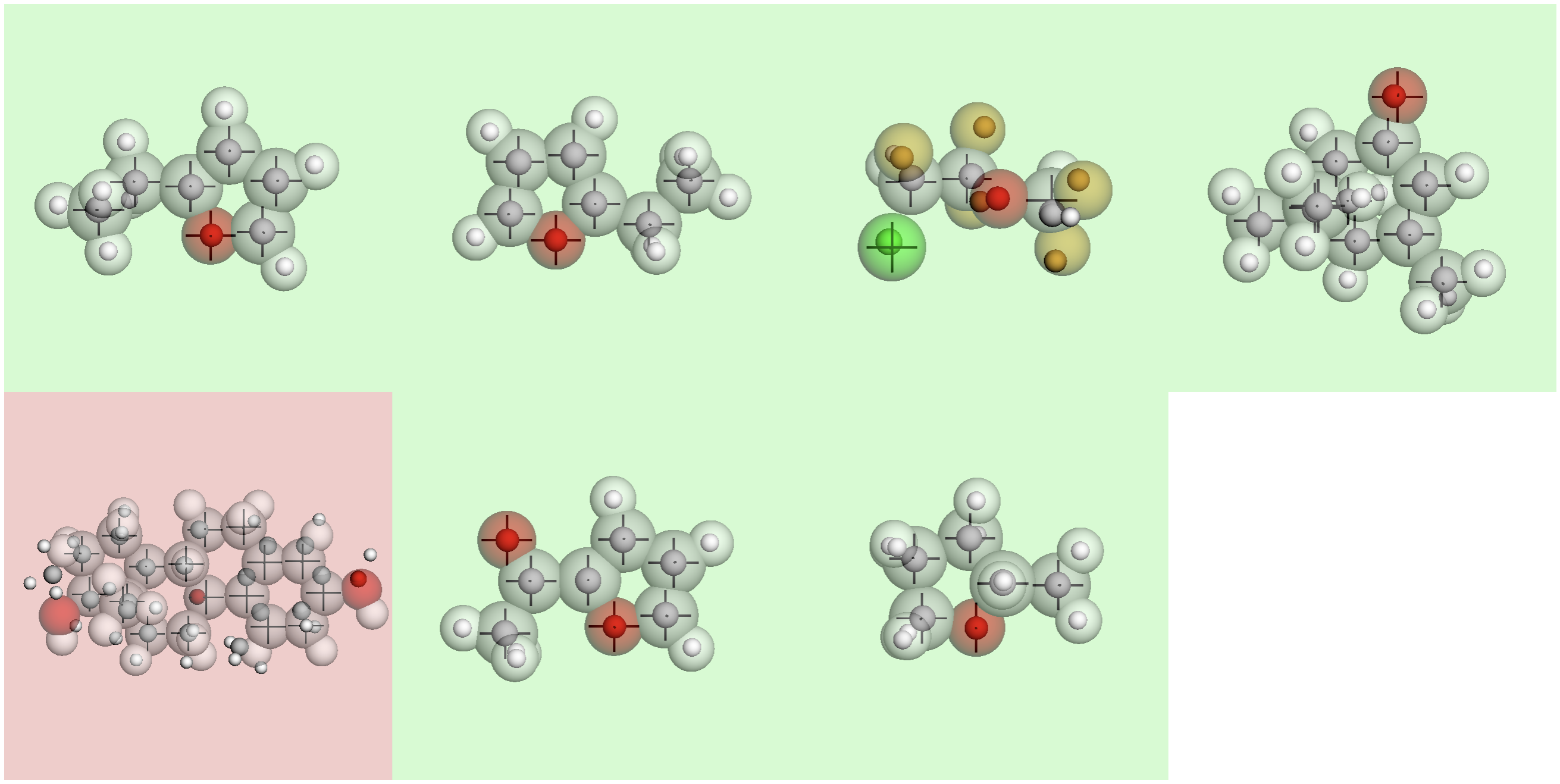}
\caption{Top-1 predictions of \name{} given experimental substitution coordinates of conformers which are very similar to examples in QM9 or GEOM.}
\label{fig:in-samples}
\end{center}
\end{figure}

\begin{figure}[H]
\begin{center}
\includegraphics[width=\textwidth]{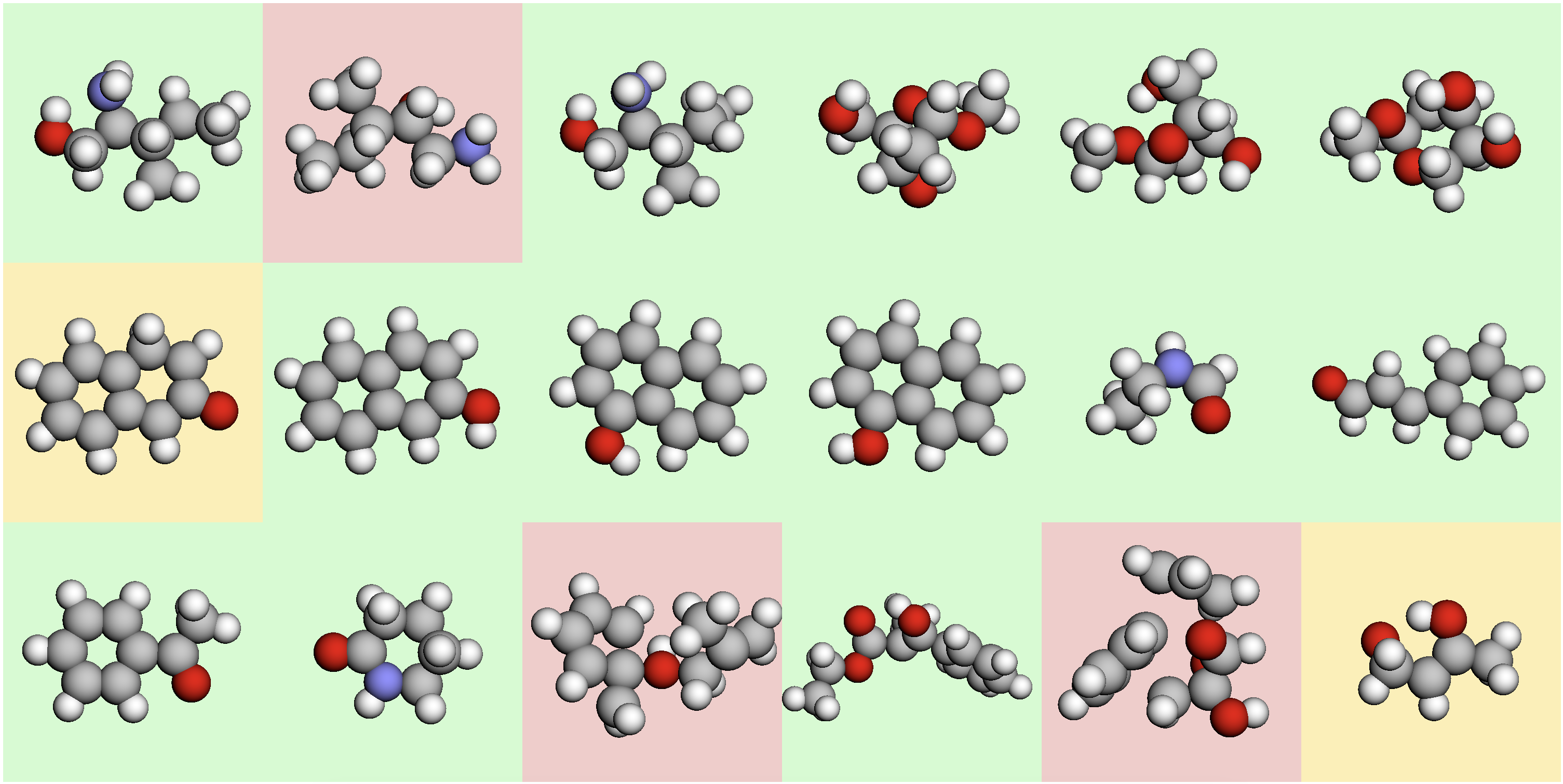}
\caption{Top-1 predictions of \name{} given experimental substitution coordinates of conformers which did not report a ground truth structure in Cartesian coordinates in their original paper. Predictions were manually verified by visual comparison to pictures in the original paper.}
\label{fig:no-gt-samples}
\end{center}
\end{figure}

\begin{table}[p]
    \centering
    \begin{tabular}{c|c|c}
\toprule
\textbf{Conformer Name} & \textbf{Reference} & \textbf{Figure} \\ 
\midrule
2,3,4-trifluorotoluene & \cite{nair2020internal} & Figure \ref{fig:ood-samples} \\
2,4,5-trifluorotoluene & \cite{nair2020internal} & Figure \ref{fig:ood-samples} \\
2'-aminoacetophenone & \cite{salvitti2022probing} & Figure \ref{fig:ood-samples} \\
cyclopropyl (isocyanato) silane (gauche) & \cite{guirgis2015molecular} & Figure \ref{fig:ood-samples} \\
cyclododecanone & \cite{burevschi2021seven} & Figure \ref{fig:ood-samples} \\
cycloundecanone & \cite{tsoi2022conformational} & Figure \ref{fig:ood-samples} \\
linalool & \cite{quesada2019analysis} & Figure \ref{fig:ood-samples} \\
a-pinene oxide & \cite{neeman2023gas} & Figure \ref{fig:ood-samples} \\
2-ethylfuran (C1) & \cite{nguyen2020heavy} & Figure \ref{fig:in-samples} \\
2-ethylfuran (Cs) & \cite{nguyen2020heavy} & Figure \ref{fig:in-samples} \\
enflurane (G+) & \cite{perez2016conformational} & Figure \ref{fig:in-samples} \\
verbenone & \cite{marshall2017rotational} & Figure \ref{fig:in-samples} \\
beta-estradiol (tg(+)) & \cite{zinn2018flexibility} & Figure \ref{fig:in-samples} \\
2-acetylfuran (anti) & \cite{dindic2021equilibrium} & Figure \ref{fig:in-samples} \\
2-methyltetrahydrofuran (equatorial) & \cite{van2016heavy} & Figure \ref{fig:in-samples} \\
isoleucinol (I) & \cite{fatima2020benchmarking} & Figure \ref{fig:no-gt-samples} \\
isoleucinol (II) & \cite{fatima2020benchmarking} & Figure \ref{fig:no-gt-samples} \\
isoleucinol (III) & \cite{fatima2020benchmarking} & Figure \ref{fig:no-gt-samples} \\
2-deoxy-d-riboside (af-1) & \cite{calabrese2020observation} & Figure \ref{fig:no-gt-samples} \\
2-deoxy-d-riboside (bf-1) & \cite{calabrese2020observation} & Figure \ref{fig:no-gt-samples} \\
2-deoxy-d-riboside (ap-1) & \cite{calabrese2020observation} & Figure \ref{fig:no-gt-samples} \\
trans-2-naphthol & \cite{hazrah2022structural} & Figure \ref{fig:no-gt-samples} \\
cis-2-naphthol & \cite{hazrah2022structural} & Figure \ref{fig:no-gt-samples} \\
cis-1-naphthol & \cite{hazrah2022structural} & Figure \ref{fig:no-gt-samples} \\
trans-1-naphthol & \cite{hazrah2022structural} & Figure \ref{fig:no-gt-samples} \\
N-ethylformamide (trans-ac) & \cite{ohba2005fourier} & Figure \ref{fig:no-gt-samples} \\
trans-cinnamaldehyde (s-trans-trans) & \cite{zinn2015structure} & Figure \ref{fig:no-gt-samples} \\
acetophenone & \cite{lei2019conformational} & Figure \ref{fig:no-gt-samples} \\
delta-valerolactam & \cite{bird2012chirped} & Figure \ref{fig:no-gt-samples} \\
trans-thymol-A & \cite{quesada2019analysis} & Figure \ref{fig:no-gt-samples} \\
strawberry aldehyde (t-aa) & \cite{shipman2011structure} & Figure \ref{fig:no-gt-samples} \\
strawberry aldehyde (c-sg-) & \cite{shipman2011structure} & Figure \ref{fig:no-gt-samples} \\
4-hydroxy-2-butanone (I) & \cite{li2022aqueous} & Figure \ref{fig:no-gt-samples} \\
\bottomrule
\end{tabular}
\caption{All 33 conformer examples with substitution coordinates extracted from the literature.}
\label{table:literature}
\end{table}

\end{document}